\documentclass[journal]{IEEEtran}
\usepackage{amsthm,amsmath,amsfonts}
\usepackage{algorithmic}
\usepackage{algorithm}
\usepackage{array}
\usepackage[caption=false,font=normalsize,labelfont=sf,textfont=sf]{subfig}
\usepackage{textcomp}
\usepackage{stfloats}
\usepackage{url}
\usepackage{verbatim}
\usepackage{cite}
\usepackage{xcolor}
\usepackage{comment}
\usepackage{graphicx}
\usepackage{subcaption}
\usepackage{amssymb}
\usepackage{booktabs}    
\usepackage{multirow}
\usepackage{hyperref}

\usepackage{titlesec}
\titlespacing{\section}{0pt}{*0.7}{*0.5}
\titlespacing{\subsection}{0pt}{*0.7}{*0.5}
\usepackage[font=footnotesize,skip=4pt]{caption}
\usepackage{enumitem}
\setlist{nosep, leftmargin=*}

\DeclareMathOperator{\tr}{\mathop{Trace}}
\DeclareMathOperator{\diag}{\mathop{diag}}

\newtheorem{theorem}{Theorem}
\newtheorem{lemma}{Lemma}

\newtheorem{proposition}{Proposition}
\newtheorem{remark}{Remark}

\newtheorem{definition}{Definition}

\begin{document}

\title{\huge Non-submodular Visual Attention for Robot Navigation}

\author{
 Reza Vafaee$^{\dagger}$, Kian Behzad$^{\dagger}$, Milad Siami$^{\dagger}$, Luca Carlone$^{\ddagger}$, and Ali Jadbabaie$^{\ddagger}$
\thanks{$^{\dagger}$ Department of Electrical \& Computer Engineering,
Northeastern University, Boston, MA 02115 USA (Emails: {\tt\small \{vafaee.r, behzad.k, m.siami\}@northeastern.edu}).}
\thanks{$^{\ddagger}$ Laboratory for
Information \& Decision Systems (LIDS), Massachusetts Institute of Technology,
Cambridge, MA 02139 USA (Emails: {\tt\small
\{jadbabai,lcarlone\}@mit.edu}).}
\thanks{This material is based upon work supported in part by the U.S. Office of Naval Research under Grant Award N00014-21-1-2431; in part by the U.S. National Science Foundation under Grant Award 2121121 and Grant Award 2208182; in part by the U.S. Department of Homeland Security under Grant Award 22STESE00001-03-02; in part by the DEVCOM Analysis Center and was accomplished under Contract Number W911QX-23-D0002 (RV, KB, MS). AJ acknowledges support from the DARPA Artificial Intelligence Quantified (AIQ) program. The views and conclusions contained in this document are those of the authors and should not be interpreted as representing the official policies, either expressed or implied, of the DEVCOM Analysis Center, the U.S. Department of Homeland Security, or the U.S. Government. The U.S. Government is authorized to reproduce and distribute reprints for Government purposes notwithstanding any copyright notation herein.}}

\maketitle
\allowdisplaybreaks

\begin{abstract}
    This paper presents a task-oriented computational framework to enhance Visual-Inertial Navigation (VIN) in robots, addressing challenges such as limited time and energy resources. The framework strategically selects visual features using a Mean Squared Error (MSE)-based, non-submodular objective function and a simplified dynamic anticipation model. To address the NP‐hardness of this problem, we introduce four polynomial‐time approximation algorithms: a classic greedy method with constant‐factor guarantees; a low‐rank greedy variant that significantly reduces computational complexity; a randomized greedy sampler that balances efficiency and solution quality; and a linearization‐based selector based on a first‐order Taylor expansion for near‐constant‐time execution. We establish rigorous performance bounds by leveraging submodularity ratios, curvature, and element‐wise curvature analyses. Extensive experiments on both standardized benchmarks and a custom control‐aware platform validate our theoretical results, demonstrating that these methods achieve strong approximation guarantees while enabling real‐time deployment.
\end{abstract}


\section{Introduction}

\IEEEPARstart{A}{chieving}  efficient navigation in unpredictable environments remains a significant challenge in robotics. Recent advancements in computing devices have opened up new possibilities, propelling substantial progress in this research area~\cite{yang2018grand}. These advancements enable near-real-time resolution of estimation and planning tasks in specific applications. However, navigating robots in rapidly changing environments still faces computational complexities. Despite the availability of high-performance computational units, the increasing demand for agility and autonomy necessitates more efficient onboard processing to ensure timely decision-making and execution in dynamic conditions.

One of the key challenges in robot navigation is achieving accurate visual odometry while minimizing computational resource usage. Many researchers have focused on the issue of selecting visual features (e.g.,~\cite{sala2006landmark,lerner2007landmark,mu2017two, carlone2018attention,mousavi2020estimation, pandey2024scalable}). The core idea is that, based on the robot's current state and its planned future motion (i.e., the task at hand), tracking specific features over a time horizon can provide more valuable information than tracking others. Essentially, some visual features may require greater attention than others.

In this context,~\cite{davison2005active} employs a greedy approach to select a subset of pre-identified visual features, thereby streamlining the robot's pose estimation. In~\cite{strasdat2009landmark}, the authors combine simultaneous localization and mapping (SLAM) using unscented Kalman filtering with reinforcement learning to create policies for feature selection.~\cite{mu2017two} proposes a two-stage methodology for measurement planning, where the initial stage involves selecting a subset of landmarks for observation, followed by determining observation times for each feature. Additionally, in~\cite{lerner2007landmark}, a task-aware approach is investigated to select a subset of features with the goal of minimizing an uncertainty metric.

Related sparsification problems are explored in other contexts. For example,~\cite{siami2017growing,siami2018network} address optimizing feedback interconnections in linear consensus networks, while~\cite{siami2020deterministic,siami2020separation} focus on selecting sparse sets of sensors or actuators to preserve system observability and controllability. However, these frameworks are not directly applicable to the motion-aware feature selection problem discussed here, as they focus on dyads rather than selecting a limited number of Positive Semidefinite (PSD) matrices.

Prior works such as~\cite{carlone2018attention,mousavi2020estimation} show that feature selection for robot navigation can be framed as maximizing a non-negative monotone submodular function under a matroid constraint, with the greedy algorithm achieving an optimal approximation of $1-1/e \approx 0.632$ for cardinality constraints~\cite{nemhauser1978analysis}. This process accelerates with lazy evaluations~\cite{badanidiyuru2014fast} and further optimizes with randomization for linear time complexity~\cite{mirzasoleiman2015lazier}. Recently,~\cite{buchbinder2024deterministic} propose a deterministic non-oblivious local search algorithm with a $1 - 1/e - \varepsilon$ (for any $\epsilon > 0$) guarantee. However, these methods are not applicable to our problem, since MSE, our primary performance metric, is not submodular.

The most closely related prior work to the present study is conducted by~\cite{carlone2018attention} and~\cite{mousavi2020estimation}. In~\cite{carlone2018attention}, the focus is on visual-inertial navigation, in which the design variable is the selection of features to track over a fixed time horizon. The authors use convex relaxations and a greedy method for feature selection, providing a quantitative assessment of performance guarantees for the quality of the resulting state estimations. They evaluate the effectiveness of the greedy heuristic using the submodularity ratio from~\cite{das2011submodular}. However, calculating the submodularity ratio for the proposed set functions is computationally challenging due to the combinatorial complexity involved in its definition~\cite{summers2019performance}. Additionally, the complexity of a simple semidefinite programming (SDP) convex relaxation scales cubically with the number of detected features, making it infeasible for real-time implementation. Moreover, their measures used to quantify the confidence ellipsoid of the forward predictor are not explicitly related to MSE, which is often the performance metric of interest for estimation problems~\cite{hashemi2018randomized}.

\begin{figure*}[tb]
    \centering
    \includegraphics[width=0.95\linewidth]{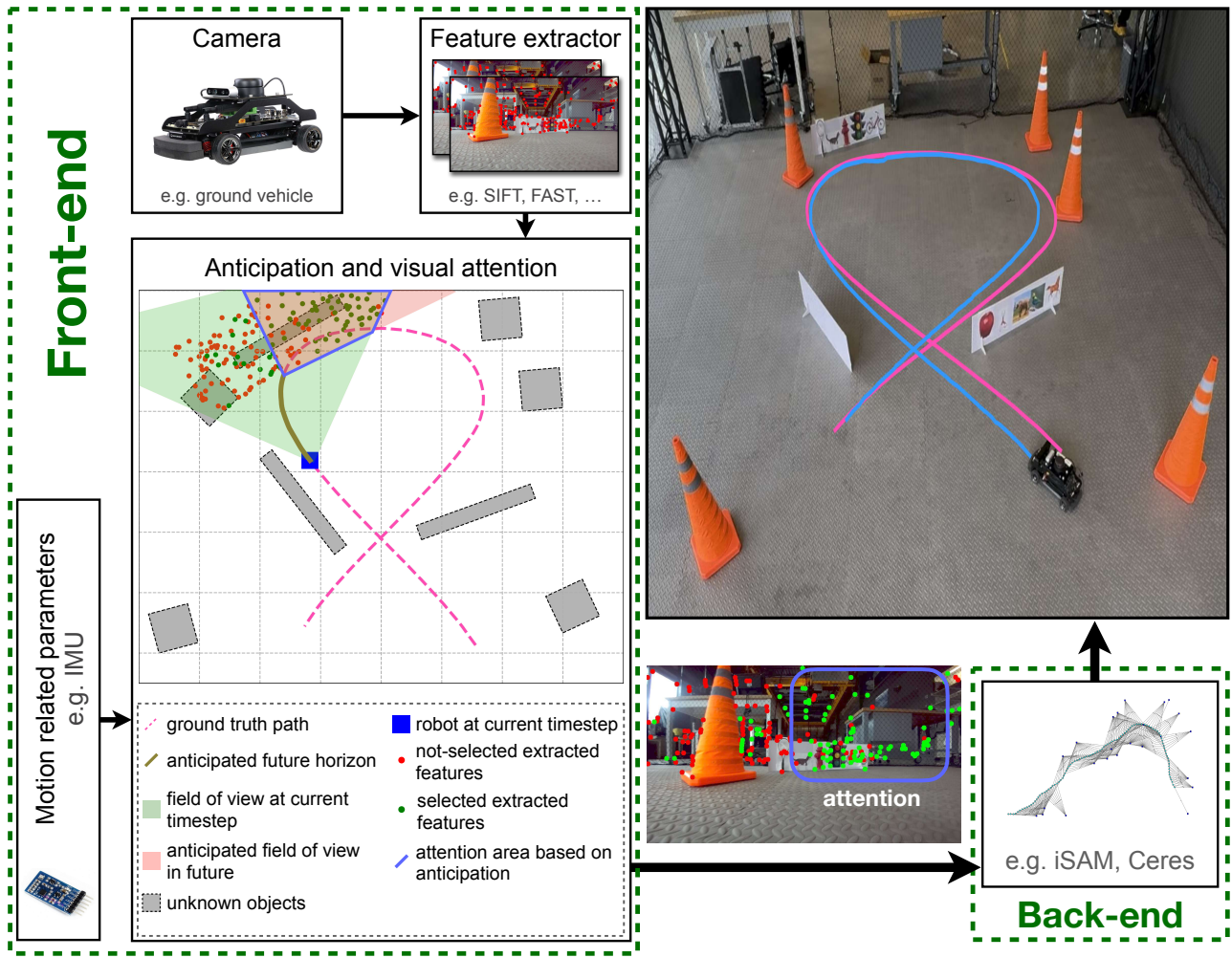}
    \caption{Detailed illustration of the visual attention mechanism as the front‐end’s filtering block. Image features extracted from each keyframe are combined with motion priors provided by the IMU to form a baseline information estimate. For every candidate feature, the mechanism projects the feature into predicted future poses to simulate its visibility and then computes its expected reduction in estimation error. Features are then ranked by this information‐gain metric, and only the highest‐value features are forwarded to the back‐end optimizer for state estimation. This selective process reduces computational load while maintaining high‐accuracy pose estimates.}
    \label{fig:diagram}
\end{figure*}

A recent study detailed in~\cite{zhao2020good} also tackles the challenge of reducing latency in Visual Odometry (VO)/Visual SLAM (VSLAM) systems by identifying and matching a subset of features deemed most valuable for pose estimation. While their focus lies in optimizing feature matching latency using the Max-logDet metric, our work concentrates on selecting informative features that minimize MSE under motion-aware dynamics, providing theoretical guarantees under non-submodular objectives.

In~\cite{mousavi2020estimation}, the authors propose a randomized sampling algorithm for feature selection instead of using the greedy method or convex relaxations. In their approach, a sampling probability (a number between zero and one) is assigned to each available feature, with these probabilities interpreted as measures of informativeness during the sampling process. This procedure provides a performance guarantee compared to using all visual features, rather than the optimal set, for a range of VIN measures, including MSE. This is achieved by approximating the complete spectrum of the information matrix. However, for some measures, such as the worst-case error, this approach imposes unnecessary computational burden because we only need to approximate the minimum eigenvalue of the information matrix, rather than the full spectrum. Additionally, the proposed approach requires sampling $\mathcal{O}((3 T + 3) \log (3 T + 3)/\epsilon^2)$ features, where $T$ is the forward time horizon and $\epsilon \in (0,1)$ is an approximation parameter, to guarantee the proposed performance bounds. We observe that this requirement for the number of sampled features is significant and, in practical scenarios, necessitates sampling all features extracted in the estimation problems.

The authors of~\cite{pandey2024scalable} later address the problem of sparse feature selection for localizing a team of agents, where they exchange relative measurements leading to a graphical network. Compared to~\cite{mousavi2020estimation}, they improve the probabilistic bound of the randomized feature selection algorithm, although the related problems mentioned in the previous paragraph still exist.

VIN methods are broadly categorized into filtering-based approaches, fixed-lag smoothing, and full smoothing. Filtering methods, such as the extended Kalman filter-based and its variants analyzed in~\cite{hesch2014camera}, offer real-time performance by sequentially processing data. However, they typically underutilize cross-time correlations, which may limit estimation accuracy in complex scenarios.

Fixed-lag smoothers~\cite{leutenegger2015keyframe,sibley2010sliding} enhance estimation accuracy via sliding-window optimization. Specifically,~\cite{leutenegger2015keyframe} integrates visual and inertial data in a keyframe-based bundle adjustment framework for visual-inertial SLAM. In contrast,~\cite{sibley2010sliding} employs a constant-time sliding-window filter using stereo imagery for real-time relative pose estimation during planetary landing. While both methods retain recent poses and marginalize older ones, their objectives and sensing modalities differ.

Full smoothing methods~\cite{forster2015imu,lupton2011visual} jointly optimize over all past states using nonlinear optimization. The first achieves efficient maximum-a-posteriori estimation through IMU preintegration on manifolds, while the second handles high-dynamic motion without requiring prior state initialization. These methods offer increased accuracy but incur higher computational complexity due to global optimization.

More recently, ORB-SLAM3~\cite{campos2021orb} has emerged as a leading keyframe-based SLAM system that integrates visual and inertial data via local bundle adjustment and global pose graph optimization. While it does not implement full smoothing over all states, its flexibility and support for multiple sensor configurations make it a strong practical baseline. This paper focuses on the complementary task of selecting informative visual measurements, which can enhance the front-end performance of VIN systems across all categories.

{\it Our contributions}: To address the existing issues, in this work,

\begin{itemize}
    \item [-] using a similar simplified model for forward simulation of robot dynamics as~\cite{carlone2018attention}, we formulate the task of feature selection as maximizing a monotone non-submodular objective function directly related to the mean-square state estimation error.
    
    \item [-] we propose constant factor approximation bounds for the greedy algorithm based on recent concepts of submodularity ratio and curvature~\cite{bian2017guarantees}. We derive bounds on these values according to the spectrum of the information matrices, eliminating the need for combinatorial search. These bounds suggest an easy-to-obtain performance guarantee for the greedy approach.

    \item [-] we exploit the low-rank structure of the landmark information matrices to significantly reduce the computational cost of matrix inversion in each greedy iteration. This leads to a fast version of the greedy algorithm that maintains the same approximation guarantees as the standard approach, while offering substantial practical speedups.
 
    \item [-] inspired by~\cite{hashemi2020randomized}, we include a randomization step in a simple greedy procedure to increase computational efficiency and practicality for scenarios with numerous detected features or longer prediction horizons. This randomized greedy framework is supported by a performance bound based on the element-wise curvature concept.
    
    \item [-] we transform the problem into an efficient modular maximization by using a first-order Taylor series approximation of the VIN performance measure. In this form, the greedy selection method guarantees an optimal solution.

    \item [-] we validate the proposed selectors in two complementary experiments.  First, multiple EuRoC sequences~\cite{burri2016euroc} allow direct comparison with prior work.  Because EuRoC lacks control inputs for horizon prediction, we also run a full visual–inertial experiment on the QCar platform~\cite{QCar_Quanser}, where stereo images, IMU data, and control commands are all available.  This two-pronged evaluation demonstrates both benchmark accuracy and real-world applicability.
\end{itemize}

Unlike~\cite{mousavi2020estimation}, all proposed performance bounds characterize proximity to the optimal visual feature set. We provide some of the proofs in Appendix~\ref{sec:missing_proofs}, to keep the main text focused on the discussion of the problem and proposed solutions.

Fig.~\ref{fig:diagram} provides a comprehensive overview of the visual attention mechanism investigated in this paper. This mechanism functions as a critical sub-block in the front-end, where it processes inputs such as extracted features from keyframes and motion-related parameters obtained, for example, from an IMU. The primary objective of the visual attention mechanism is to determine which features should be selected for back-end optimization. This selection process is guided by evaluating both the quality of the features and their potential impact in decreasing uncertainties based on robot dynamics. By effectively filtering and prioritizing features, the visual attention mechanism ensures that the back-end optimization process operates with the most relevant and high-quality data, thereby enhancing the overall performance and accuracy of the system.

\section{Preliminaries and Problem Definition}
\subsection{Notations}
Integers are denoted by $\mathbb{Z}$ and real numbers by $\mathbb{R}$. 
The set of integers (respectively real numbers) greater than or equal to $a \in \mathbb{R}$ is denoted by $\mathbb{Z}_{\geq a}$ (respectively $\mathbb{R}_{\geq a}$). 
Given any integer~\(n \in \mathbb{Z}_{\geq 1}\), we define~\([n] = \{1, \ldots, n\}\). Finite sets are denoted by sans-serif fonts (e.g.,~\(\mathsf{A}\)). The cardinality of a finite set~\(\mathsf{A}\) is denoted by~\(|\mathsf{A}|\). 
\(I\) denotes an identity matrix, with its dimension inferred from context. The transpose and rank of a matrix~\(X\) are denoted by~\(X^\top\) and~\(\operatorname{rank}(X)\), respectively. 
The set of positive definite (respectively positive semidefinite) matrices of size~\(n\) is denoted by~\(\mathbb{S}_{++}^n\) (respectively~\(\mathbb{S}_{+}^n\)). 
The symbol~\(\|\cdot\|_2\) denotes the Euclidean norm when applied to vectors and the spectral norm when applied to matrices. The symbol~\(\|\cdot\|_F\) denotes the Frobenius norm for matrices.
For vector~\(x\),~\(\diag x\) is a diagonal matrix with the elements of~\(x\) orderly sitting on its diagonal.~\(\mathcal{N}(\mu, \Sigma)\) denotes a normal (Gaussian) distribution with mean (average)~\(\mu\) and covariance~\(\Sigma\).

\subsection{Forward-simulation Optimization}
The forward-simulation model considered in this paper depends on the anticipated future motion of the robot, as we will later show that the Inertial Measurement Unit (IMU) and vision models are functions of the robot's predicted future poses. Anticipation is a key element of our approach: the feature selection mechanism is aware of the robot's immediate future intended motion and selects features that remain in the field of view longer given these intentions~\cite{carlone2018attention}.

In typical VIN pipelines, the primary objective is to estimate the state of the robot at each frame. Let~\(k\) denote the time frame for estimation, where our focus is on selecting a limited number of features. Let~\(x(k)\) represent the state of the robot at time~\(k\). For now, readers can consider~\(x(k)\) to include the pose and velocity of the robot at time instant~\(k\), as well as the IMU biases. We denote
\[
    \hat{x}(k:k+T) \triangleq
    \begin{bmatrix}
    \hat{x}^\top(k) & \hat{x}^\top(k+1) & \cdots & \hat{x}^\top(k+T)    
    \end{bmatrix}^\top,
\]
as the future state estimates within the horizon~\(T\). Moreover, we denote~\(P\) as the covariance matrix of our estimate~\(\hat{x}(k:k+T)\), and~\(\Omega \triangleq P^{-1}\) as the corresponding Fisher information matrix. The error variance for the MSE estimator~\(\hat{x}(k:k+T)\) can then be obtained by summing the variances of its individual scalar components, i.e.,
\begin{equation}\label{eq:mse}
    \mathbb{E} \, \big \|x(k:k+T) - \hat{x}(k:k+T) \big \|_2^2 = \tr \Omega^{-1},    
\end{equation}
where~\(x(k:k+T)\) is the stacked vector of actual values of the state vectors over the horizon~\(T\).

Let~\(\mathsf{U}\), with~\(|\mathsf{U}| = N\), denote the set of all available features at time frame~\(k\). Our goal is to select a limited number of~\(\kappa \ll N\) features that maximize
\begin{equation}\label{problem::main}
     \max_{\mathsf{S} \, \subset \, \mathsf{U}} \, f(\mathsf{S}) \quad \text{subject to} \quad |\mathsf{S}| \leq \kappa.
\end{equation}
with~\(f(\mathsf{S}) \triangleq \tr \Omega_{\emptyset}^{-1} - \tr \Omega_{\mathsf{S}}^{-1}\),~\(\Omega_{\mathsf{S}} \triangleq \Omega_{\emptyset} + \sum_{l \in \mathsf{S}} \Delta_l\), where~\(\Omega_{\emptyset}\) is the information matrix of the estimate when no features are selected, and~\(\Delta_l\) is the information matrix associated with the selection of the~\(l\)-th visual feature.

It is evident that the maximization problem~\eqref{problem::main} aims to minimize the MSE of the estimate as defined in~\eqref{eq:mse}. Furthermore, we incorporate the term~\(\tr \Omega_{\emptyset}^{-1}\) in the objective of~\eqref{problem::main} to ensure the normalization of the objective, such that it returns zero for empty sets.
It is important to note that the MSE-based objective $f$ is not necessarily submodular, as demonstrated by counterexamples in~\cite{olshevsky2017non}.
We note that similar criteria have been developed in~\cite{vafaee2024learning, vafaee2023exploring, hashemi2020randomized} for sparse sensing purposes.

At first glance, the problem definition in~\eqref{problem::main} may appear similar to the often-studied sensor selection problem for Kalman filtering, such as in~\cite{vafaee2023exploring},~\cite{hashemi2020randomized}, or~\cite{chamon2020approximate}. However, the contribution of each visual feature to the information matrix, i.e., the~\(\Delta_{l}\) matrices, as we will see, is PSD and typically of higher rank. In contrast, each sensor's contribution to the information matrix of Kalman state estimation is generally a rank-one matrix, often in the form of dyads. This key difference in the structure of the information matrices imposes challenges that preclude directly applying similar algorithms from sensor selection contexts to the visual feature selection problem.

\begin{remark}
    The authors of~\cite{buchbinder2024deterministic} have recently developed a significant advancement in submodular maximization under a matroid constraint, introducing a deterministic non-oblivious local search algorithm with an approximation guarantee of $1 - 1/e - \varepsilon$ (for any $\varepsilon > 0$). This approach effectively bridges the gap between deterministic and randomized algorithms for such problems. However, their method is specifically designed for submodular functions. Since our problem~\eqref{problem::main} involves the MSE measure, which is non-submodular, their deterministic algorithm cannot be directly applied to achieve the desired optimization in our context.
\end{remark}

In the subsequent two sections, we discuss how the integration of IMU and camera measurements can be utilized to determine the information matrices~\(\Omega_{\emptyset}\) and~\(\Delta_l\) at each time frame.

\subsection{IMU Model}
There is often a significant discrepancy in measurement frequencies, with the IMU operating at a much higher frequency than the camera.
Consequently, following the methodology outlined in~\cite{carlone2018attention}, we integrate a set of IMU measurements between two consecutive camera frames~\(k\) and~\(j\), treating this integrated measurement as a single measurement that constrains the states~\(x(k)\) and~\(x(j)\) using a linear measurement model.

Our simplified IMU model assumes that the accumulation of rotation error through gyroscope integration over time is negligible. Given this assumption of accurate rotation estimates from the gyroscope, we reduce the state variables to the robot's position, linear velocity, and accelerometer bias.

The on-board accelerometer measures the acceleration~\(a(k)\) of the sensor with respect to the inertial frame. This measurement is affected by additive white noise~\(\xi(k)\) and a slowly varying sensor bias~\(b(k)\). Therefore, the measurement~\(\tilde{a}(k) \in \mathbb{R}^3\) acquired by the accelerometer at time~\(k\) is given as
\begin{equation}\label{eq:IMU}
    \tilde{a}(k) = R^\top(k) \left(a(k) - g\right) + b(k) + \xi(k),
\end{equation}
where~\(g\) is the gravity vector expressed in the inertial frame, and~\(R(k)\) denotes the attitude of the robot at time~\(k\), assumed to be known from gyroscope integration over the horizon.

Using~\eqref{eq:IMU}, we can show that the vector of IMU measurements between frame~\(k\) and the future frame~\(j\) exhibits linear dependence on the stack of state vectors~\(x(k:k+T)\). By applying these linear models to all consecutive frames~\(k\) and~\(j\) within the horizon, we can deduce the information matrix for the estimate of~\(x(k:k+T)\) based on the IMU data.

Note that the information matrix obtained only from relative measurements lacks a constraint anchoring the information to a global reference, making it rank deficient. To address this deficiency, we integrate a prior on the state at time~\(k\) from the VIN back-end into the information matrix. This incorporation yields a Positive Definite (PD) information matrix~\(\Omega_{\emptyset} \in \mathbb{S}_{++}^{9T+9}\). This PD matrix is essential for our problem definition in~\eqref{problem::main}.

The derivation of the information matrix~\(\Omega_{\emptyset}\) follows the same process as in~\cite[III-B1]{carlone2018attention}. To avoid repetition, detailed derivations are omitted in this paper. For a comprehensive derivation, please refer to the relevant section of the cited paper.

\subsection{Vision Model}
We use a linearized version of the nonlinear perspective projection model as the vision model. In this approach, we represent a pixel measurement as a linear function of the unknown state we aim to estimate. A calibrated pixel measurement of an external 3-D point (or feature)~\(l\) provides the 3-D bearing of the visual feature in the camera frame.

Mathematically, let~\(u_l(k)\) be the unit vector corresponding to the (calibrated) pixel observation of feature~\(l\) from the robot's pose at time~\(k\). This unit vector~\(u_l(k)\) satisfies the following relation:
\begin{equation}\label{eq:vision}
    u_l(k) \times \left( \left(R^{\text{W}}_{\text{cam}}(k)\right)^\top \left(p_l - t^{\text{W}}_{\text{cam}}(k)\right)\right) = \mathbf{0}_3,
\end{equation}
where~\(\times\) denotes the cross product between two vectors,~\(p_l\) is the 3-D position of visual feature~\(l\) in the world frame, and~\(R^{\text{W}}_{\text{cam}}(k)\) and~\(t^{\text{W}}_{\text{cam}}(k)\) are the rotation and translation describing the camera pose at time~\(k\) with respect to the world frame.

Since our state at each time step~\(k\) includes the position, velocity of the robot in the IMU frame, and accelerometer bias, we reparametrize~\eqref{eq:vision} to incorporate these variables. This reparametrization uses the known extrinsic transformation between the camera and IMU obtained from calibration. It also enables us to derive a linear model for the calibrated feature pixel as a function of the state vector~\(x(k:k+T)\) and~\(p_l\), the unknown position of feature~\(l\).

In the context of the obtained linear measurement model, our objective is to determine the visibility of each feature~\(l\) at future time points over the horizon~\(T\). The key idea is that if the~\(l\)-th feature is expected to remain in view throughout the future horizon, it will provide more information compared to a feature that is anticipated to be quickly obscured. To achieve this, we project the~\(l\)-th feature onto the image plane of the robot's camera at future time instances. Ultimately, our task is to verify whether the feature~\(l\) will be visible in the forward predictions. Forward simulation and visibility checking are illustrated in Fig.~\ref{fig:visibility}.

\begin{figure}
\centering
\includegraphics[width=0.9\columnwidth]{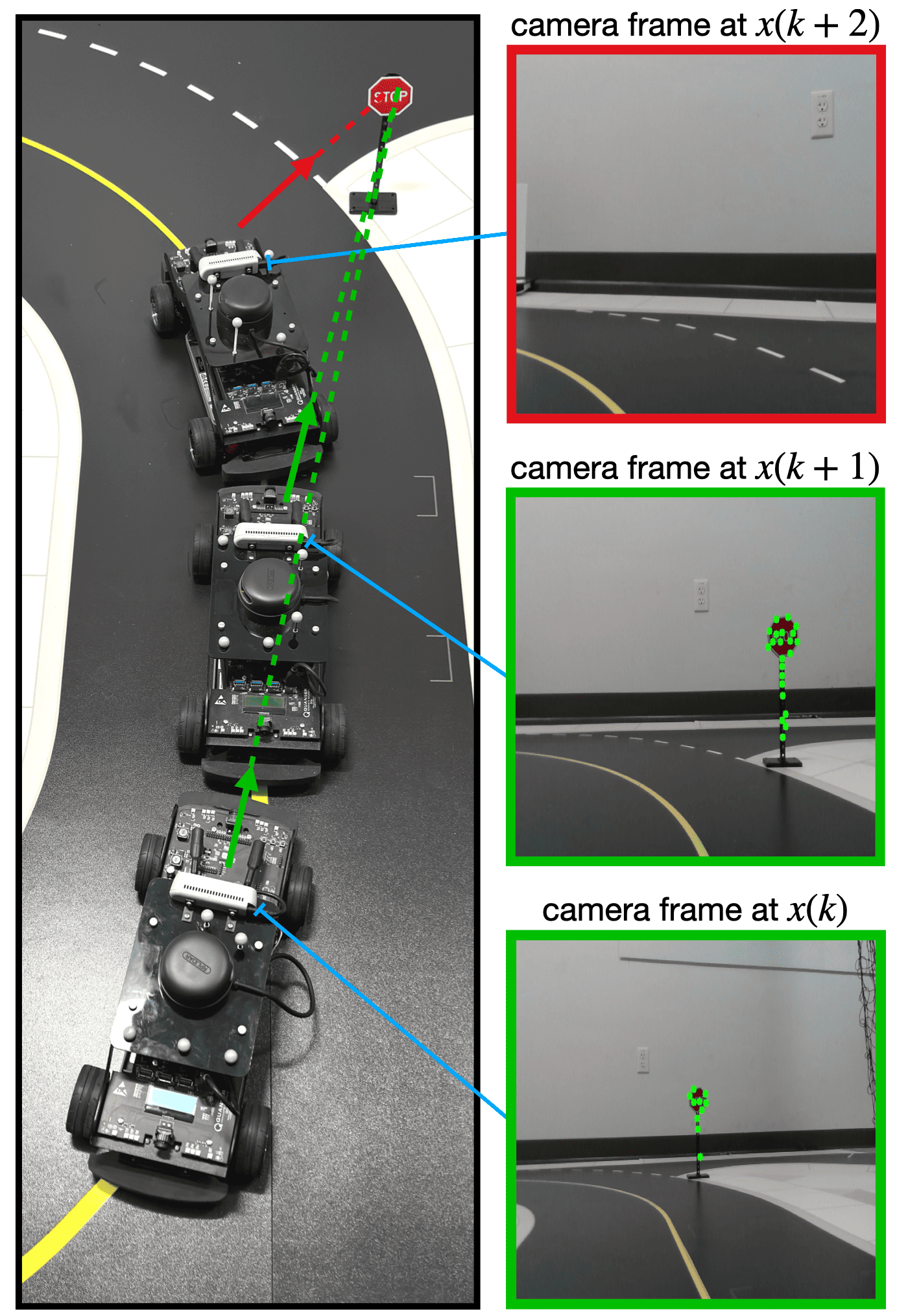} 
    \caption{Forward propagation of a feature’s bearing vector from a 3-D object (a stop sign) detected at frame~\(k\). Left: the robot’s positions and bearing vectors to the stop sign features at time steps~\(k+1\) and~\(k+2\). Right: the corresponding camera image planes at those time steps. Colored rectangles around each image indicate whether the feature is visible (green) or not (red). In this example, the predicted bearing at~\(k+2\) falls outside the camera’s field of view, causing the visibility check to fail. See Section~\ref{sec:experiments:QCar} for details.}
    \label{fig:visibility}
\end{figure}

We then stack the linear measurement model for each observation pose from which feature~\(l\) is visible. It is important to note that the stacked linear model cannot be directly used to estimate our state vector~\(x(k:k+T)\) because it includes the unknown position of feature, $p_l$. To address this challenge, we employ the Schur complement trick~\cite{carlone2014eliminating} to analytically eliminate the 3D point from the estimation process.
To achieve this, we first construct the information matrix for the joint state vector~\(\begin{bmatrix} x(k:k+T)^\top & p_{l}^\top \end{bmatrix}^\top\) using the stacked linear measurement model:
\begin{equation}
    \Omega^{(l)} = \begin{bmatrix}
        F_l^\top F_l & F_l^\top E_l\\
        E_l^\top F_l & E_l^\top E_l
    \end{bmatrix} \in \mathbb{R}^{(9(T+1)+3) \times (9(T+1) + 3)}.
\end{equation}

For detailed definitions of the matrices~\(F_l\) and~\(E_l\), we refer interested readers to~\cite{carlone2018attention}. By applying the Schur complement trick, we marginalize out the feature~\(l\) and obtain the information matrix for our state~\(x(k:k+T)\) given the measurements:
\begin{equation}\label{eq:delta_l}
    \Delta_l \triangleq F_l^\top F_l - F_l^\top E_l \left( E_l^\top E_l \right)^{-1} E_l^\top F_l \in \mathbb{R}^{9(T+1) \times 9(T+1)}.
\end{equation} 
This matrix~\(\Delta_l\) represents the additive contribution to the information matrix of our state estimate resulting from the measurements of a single visual landmark~\(l\). It is a PSD matrix. 

\section{Approximation Methods}
The combinatorial problem in~\eqref{problem::main} is NP-hard by reduction to the classical set cover problem; see, e.g.,~\cite{hashemi2020randomized, williamson2011design}, requiring an exhaustive search over all selections of~\(\kappa\) features to find the optimal solution. Therefore, in this section, we discuss various polynomial-time greedy-based algorithms to approximate its solution. We start with the more computationally expensive algorithms and then move on to the less demanding ones.

\begin{remark}
    The approximation algorithms proposed in this paper, along with their theoretical analyses and performance guarantees, can be extended to alternative set function objectives, such as the minimum eigenvalue or the log-determinant criterion discussed in~\cite{carlone2018attention}. In this work, however, we focus on the MSE objective due to its direct and interpretable connection to estimation accuracy, making it particularly suitable for the scenarios considered.
\end{remark}

\subsection{Greedy}
In this section, we analytically demonstrate why a simple greedy algorithm provides a reliable approximate solution to problem~\eqref{problem::main}. Our findings are based on recent literature on submodularity and submodular maximization~\cite{bian2017guarantees}.

As previously defined,~\(\mathsf{U}\) represents the ground set containing all extracted features, and~\(f(\cdot)\) denotes our set function as in~\eqref{problem::main}.
We define the marginal gain of a feature set~\(\mathsf{R} \subset \mathsf{U}\) in the context of a set~\(\mathsf{S} \subset \mathsf{U}\) as
\begin{equation}\label{eq:marginal}
    f_{\mathsf{R}}(\mathsf{S}) \triangleq f(\mathsf{R} \cup \mathsf{S}) - f(\mathsf{S}).    
\end{equation}
For future analysis, we may use the shorthand notation~\(l\) and~\(f_l(\cdot)\) for~\(\{l\}\) and~\(f_{\{l\}}(\cdot)\), respectively.

\begin{definition}
    A set function
    \begin{itemize}
        \item [-]~\( f : 2^{\mathsf{U}} \to \mathbb{R} \) is called \emph{submodular} if for all subsets~\( \mathsf{R} \subseteq \mathsf{S} \subseteq \mathsf{U} \) and all elements~\( l \notin \mathsf{S} \), it holds that
        \begin{equation}\label{eq:sub}
            f(\mathsf{R} \cup \{l\}) - f(\mathsf{R}) \geq f(\mathsf{S} \cup \{l\}) - f(\mathsf{S}).
        \end{equation}
        A set function is \emph{supermodular} if the reversed inequality in~\eqref{eq:sub} holds, and is \emph{modular} if~\eqref{eq:sub} holds with equality.

    \item [-] is monotone nondecreasing if for all~\( \mathsf{R} \subseteq \mathsf{U} \) and~\( l \in \mathsf{U} \), $f(\mathsf{R} \cup \{l\}) \geq f(\mathsf{R})$.

    \item [-] is normalized if~\( f(\emptyset) = 0 \).
    \end{itemize}
\end{definition}

Submodularity intuitively reflects the diminishing returns property: adding an element to a smaller set yields a greater benefit than adding it to a larger set. This principle is used to derive constant-factor approximation guarantees, specifically~\((1 - 1/e)\), for the greedy algorithm when applied to submodular maximization problems with cardinality constraints~\cite{nemhauser1978analysis}.

Algorithm~\ref{alg::greedy} operates by iteratively selecting the feature that provides the maximum marginal gain in reducing the estimation error, as defined by the objective function~\(f(\mathsf{S})\). At each iteration, the feature that contributes most to decreasing the MSE is added to the current subset, until the cardinality constraint~\(\kappa\) is met. However, since the set function~\(f(\cdot)\) is not submodular~\cite{bian2017guarantees}, a quantitative assessment of solutions using the simple greedy procedure (Algorithm~\ref{alg::greedy}) requires further elaboration. To address this, we need the following two technical definitions from~\cite{bian2017guarantees}.

\begin{algorithm}[tb]
    \small
    \caption{Simple Greedy Algorithm}
    \label{alg::greedy}
    \begin{algorithmic}[1]
    \STATE {\bfseries Input:} $\Omega_{\emptyset}$, $\Delta_l$ for all $l \in \mathsf{U}$, and $\kappa$
    \STATE {\bfseries Output:} subset $\mathsf{S} \subset \mathsf{U}$ with $|\mathsf{S}| = \kappa$
    \STATE $\mathsf{S} \gets \emptyset$ 
    \WHILE {$|\mathsf{S}| < \kappa$}
        \STATE $\mathsf{S} \gets \mathsf{S} \cup \{\arg \max_{l \in \mathsf{U} \setminus \mathsf{S}} \left[ f(\mathsf{S} \cup \{l\}) - f(\mathsf{S}) \right]\}$
    \ENDWHILE
    \end{algorithmic}
\end{algorithm}

\begin{definition}[Curvature, $\alpha$]\label{def:curv}
    The curvature of any nonnegative set function~\(f(\cdot)\) is the smallest scalar~\(\alpha \in \mathbb{R}_{\geq 0}\) that satisfies:
    \begin{equation}\label{eq:curve}
        f_{l}\left(\mathsf{S} \backslash \{l\} \cup \mathsf{R}\right) \geq (1 - \alpha) \cdot f_{l}(\mathsf{S} \backslash \{l\}),
    \end{equation}
    for all subsets~\(\mathsf{R}\) and~\(\mathsf{S}\) of~\(\mathsf{U}\), and for all~\(l \in \mathsf{S} \backslash \mathsf{R}\).
\end{definition}

For a nondecreasing set function~\(f(\cdot)\), the curvature~\(\alpha\) satisfies~\(\alpha \in [0, 1]\). The function~\(f(\cdot)\) is supermodular if and only if~\(\alpha = 0\).

\begin{definition}[Submodularity ratio, $\gamma$]\label{def:sub-ratio}
    The submodularity ratio of any nonnegative set function~\(f(\cdot)\) is the largest scalar~\(\gamma \in \mathbb{R}_{\geq 0}\) such that
    \begin{equation}\label{eq:submo}
        \sum_{l \in \mathsf{R} \backslash \mathsf{S}} f_{l}(\mathsf{S}) \geq \gamma \, f_{\mathsf{R}}(\mathsf{S}),
    \end{equation}
    for all~\(\mathsf{R}, \mathsf{S} \subseteq \mathsf{U}\).
\end{definition}

For a nondecreasing set function~\(f(\cdot)\), the submodularity condition holds if and only if~\(\gamma = 1\). However, in general,~\(\gamma\) lies within the interval~\([0, 1]\).

\begin{proposition}[Approximate non-submodular maximization~\cite{bian2017guarantees}]\label{lem:nonsub_bound}
    Let $f$ be a nonnegative, nondecreasing, normalized set function with submodularity ratio $\gamma \in [0, 1]$ and curvature $\alpha \in [0, 1]$. Then, Algorithm~\ref{alg::greedy}, when applied to problem~\eqref{problem::main}, provides the following approximation guarantee:
    \begin{equation}\label{eq:greedy_bound}
        f(\mathsf{S}_{\text{greedy}}) \geq \frac{1}{\alpha} \left(1 - e^{-\alpha \gamma}\right) f(\mathsf{S}_{\text{OPT}}),
    \end{equation}
    where $\mathsf{S}_{\text{greedy}}$ is the subset returned by Algorithm~\ref{alg::greedy} and $\mathsf{S}_{\text{OPT}}$ is the optimal solution to problem~\eqref{problem::main}.
\end{proposition}

To apply the results from Proposition~\ref{lem:nonsub_bound} to our feature selection problem, we must demonstrate that the objective function in~\eqref{problem::main} satisfies the conditions specified in the lemma. This verification is provided in Lemma~\ref{lem:one}.

Let $n = 9T + 9$ denote the dimension of the information matrices in the VIN problem. It is easy to see that $\Omega_{\mathsf{S}} = \Omega_{\emptyset} + \sum_{l \in \mathsf{S}} \Delta_{l}$ is a symmetric positive definite matrix. This allows for factorization as $V \, \diag 
\begin{bmatrix}
    \lambda_1(\Omega_{\mathsf{S}}) & \cdots & \lambda_{n}(\Omega_{\mathsf{S}})
\end{bmatrix}^\top \, V^\top$, according to the eigendecomposition. Consequently, one can derive~\(\tr \Omega_{\mathsf{S}}^{-1} = \sum_{i = 1}^{n} \frac{1}{\lambda_i(\Omega_{\mathsf{S}})}\).

\begin{lemma}\label{lem:one}
    The objective of~\eqref{problem::main} is nonnegative, monotone increasing, and normalized.
\end{lemma}

\begin{proof}
    The nonnegativity follows from $\Omega_{\emptyset}$ being positive definite and $\Delta_{l}$ for all $l$ being positive semidefinite. It is also straightforward to confirm that $f(\emptyset) = 0$, thus ensuring normalization. To demonstrate that the objective in~\eqref{problem::main} is monotone nondecreasing, we need to show that for all $l \in \mathsf{U} \setminus \mathsf{S}$, the inequality $f(\mathsf{S} \cup \{l\}) - f(\mathsf{S}) \geq 0$ holds. This is established as follows: $f(\mathsf{S} \cup \{l\}) - f(\mathsf{S}) =  \sum_{i=1}^{n}\big[\frac{1}{\lambda_i(\Omega_{\mathsf{S}})} - \frac{1}{\lambda_i(\Omega_{\mathsf{S} \cup \{l\}})}\big] > 0$, due to the \emph{Cauchy interlacing inequality} for singular values.
\end{proof}

Lemma~\ref{lem:one} confirms that the constant-factor approximation bound of Proposition~\ref{lem:nonsub_bound} applies to our VIN problem. The challenge, however, lies in the intractability of computing the submodularity ratio and curvature for a given set function due to the combinatorial number of constraints (of the order of $2^{2|\mathsf{U}|}$) in~\eqref{eq:curve} and~\eqref{eq:submo}, respectively. This difficulty is akin to the exhaustive search required for solving problem~\eqref{problem::main}. Nonetheless, establishing a positive lower bound on the submodularity ratio and an upper bound on the curvature for $f$ justifies the use of the greedy algorithm for problem~\eqref{problem::main} via Proposition~\ref{lem:nonsub_bound}. Our objective is to derive such bounds for $f$ corresponding to the non-submodular objective in~\eqref{problem::main}.

\begin{theorem}\label{thrm:one}
    The Greedy Algorithm~\ref{alg::greedy} provides the following approximation guarantee for the feature selection problem~\eqref{problem::main}:
    \begin{equation}\label{eq:theorem_one} 
        f(\mathsf{S}_{\text{greedy}}) \geq \frac{1}{\alpha} (1 - e^{-\alpha \gamma}) f(\mathsf{S}_{\text{OPT}}) \geq \frac{1}{\overline{\alpha}} \left(1 - e^{-\overline{\alpha}\underline{\gamma}}\right) f(\mathsf{S}_{\text{OPT}}),
    \end{equation}
    where $\overline{\alpha}$ and $\underline{\gamma}$ are the upper bound on the curvature and the lower bound on the submodularity ratio of the objective in~\eqref{problem::main}, respectively. These bounds can be obtained using the spectrum of the information matrices at each time frame of the VIN pipeline as follows:
    \begin{equation}\label{eq:mes_one}
        \overline{\alpha} = 1 - \underline{\gamma} \, = \, 1 - \frac{\delta \cdot \lambda_{\min}(\Omega_{\emptyset})}{\lambda_{\max}(\Omega_{\mathsf{U}}) \cdot (\lambda_{\max}(\Omega_{\mathsf{U}}) - \lambda_{\min}(\Omega_{\emptyset})},
    \end{equation}
    where $\Omega_{\emptyset}$ is the information matrix of the estimate when no features are selected, $\Omega_{\mathsf{U}}$ is the information matrix when all features are utilized, and $\delta$ is a positive number such that $\tr \Delta_l \geq \delta$ for all $l \in \mathsf{U}$.
\end{theorem}

\begin{proof}
    See Appendix~\ref{sec:missing_proofs}.
\end{proof}

The authors of~\cite{carlone2018attention} utilize the concept of the submodularity ratio, as proposed by~\cite{das2011submodular}, to evaluate the effectiveness of a simple greedy heuristic in addressing a problem similar to problem~\eqref{problem::main}. However, determining the submodularity ratio for a given set function remains intractable due to the combinatorial number of constraints inherent in the definition presented by~\cite{das2011submodular}. In our results, as shown in Theorem~\ref{thrm:one}, we also incorporate the notion of curvature to derive an improved constant-factor approximation. To avoid the combinatorial search required for obtaining these quantities, we present a set of computable bounds on the submodularity ratio and curvature.

In our analysis, we establish the relationship $\overline{\alpha} = 1 - \underline{\gamma}$. This correlation arises from the conservative approach taken in formulating these bounds. However, it is important to emphasize that this connection is specific to the derived values and does not generalize to the broader context of submodularity ratio and curvature values. For instance, for a submodular function, the submodularity ratio ($\gamma$) is equal to 1. However, this does not imply that the curvature ($\alpha$) is necessarily $\alpha \leq 1 - 1 = 0$. Such an implication would hold true only if the function were also supermodular and, consequently, modular.
Despite the conservative nature of these bounds, their existence is promising and aligns with empirical observations regarding the effectiveness of the greedy algorithm in optimizing non-submodular set functions.

It is worth emphasizing that the values of curvature and submodularity ratio are not intrinsic properties of the set function alone, as they also depend on the choice of ground set and the budget parameter~\(\kappa\). In other words, even for a fixed function~\(f(\cdot)\), the computed values of~\(\alpha\) and~\(\gamma\) may vary depending on the context in which they are evaluated. This implies that the approximation guarantee given in Proposition~\ref{lem:nonsub_bound} is meaningful only after verifying that~\(\gamma > 0\) for the particular setup at hand. In practice, this can often be assessed numerically, as we demonstrate in our empirical results (e.g., Fig.~\ref{fig:bound}(a)), where a nonzero submodularity ratio validates the use of the greedy algorithm in our VIN pipeline.

In greedy Algorithm~\ref{alg::greedy}, we iteratively examine all feature candidates and find the one whose addition will enhance the estimation quality the most. This method requires $\mathcal{O}(\kappa N T^3)$ operations, implying that the complexity of computation grows quadratically with the number of available features if, in the worst-case, $\kappa = \mathcal{O}(N)$, and cubically with the length of the time horizon. Therefore, the greedy algorithm may be less practical in scenarios where the number of detected features or the length of the considered horizon is substantial.

After marginalizing out the 3-D position of landmark~$l$, the expression in~\eqref{eq:delta_l} can be rewritten in a structured, low-rank form:
\begin{equation}\label{eq:delta_decom}
    \Delta_l = F_l^{\top} Q_l F_l, \qquad Q_l \triangleq I - E_l \left(E_l^{\top}E_l\right)^{-1}E_l^{\top},
\end{equation}
where $F_l \in \mathbb{R}^{3n_\ell \times 9(T+1)}$ and $E_l \in \mathbb{R}^{3n_\ell \times 3}$. Here, $n_\ell$ denotes the \emph{number of frames in which landmark~$l$ is observed}, which is readily available via visibility checks (see~\cite{carlone2018attention} for details). This decomposition enables a computational shortcut: instead of recomputing the matrix inverse $(\Omega_{\mathsf{S}} + \Delta_l)^{-1}$ at each iteration, we maintain $\Omega_{\mathsf{S}}^{-1}$ and apply the Sherman–Morrison–Woodbury (SMW) identity. Since the update cost scales with $\operatorname{rank}(\Delta_l)^2 \cdot n$ rather than $n^3$, this yields a practical ``fast greedy'' algorithm with a per-iteration complexity of approximately $\mathcal{O}(\kappa N n_\ell^2 T)$, compared to the original $\mathcal{O}(\kappa N T^3)$. The next section details this efficient implementation.

\subsection{Fast Low-Rank Greedy}\label{subsec:fast_greedy}
Selecting the top~$\kappa$ landmarks from a set of $N$ candidates is appealing in theory, but quickly becomes computationally infeasible in practice. Even for a modest time horizon $T = 20$, the matrix dimension becomes $n = 9(T+1) = 189$, and inverting a dense $189 \times 189$ matrix for each candidate landmark at every greedy step entails a cost of $\mathcal{O}(n^3) \approx 7 \times 10^6$ flops. This makes real-time decision-making impractical when $N$ is large.

To overcome this bottleneck, we leverage the low-rank structure of each landmark update $\Delta_l$ derived in~\eqref{eq:delta_decom}. By applying the SMW identity and maintaining the inverse $\Omega_{\mathsf{S}}^{-1}$ throughout the algorithm, each matrix inverse can be updated efficiently with complexity $\mathcal{O}(n \cdot n_\ell^2)$, where $n_\ell$ is the number of frames in which landmark $l$ is observed.

The efficient update for a candidate $l \notin \mathsf{S}$ proceeds by computing
\begin{equation}\label{eq:SMW}
    (\Omega_{\mathsf{S}} + \Delta_l)^{-1} = \Omega_{\mathsf{S}}^{-1} - \Omega_{\mathsf{S}}^{-1} F_l^{\top} \left(Q_l^{-1} + F_l \Omega_{\mathsf{S}}^{-1} F_l^{\top} \right)^{-1} F_l \Omega_{\mathsf{S}}^{-1},
\end{equation}
where $Q_l = I - E_l (E_l^{\top} E_l)^{-1} E_l^{\top}$. Substituting~\eqref{eq:SMW} into the objective function, the marginal information gain $f(\mathsf{S} \cup \{l\}) - f(\mathsf{S})$ becomes 
\begin{equation}\label{eq:marginal_info}
    \tr \left( \Omega_{\mathsf{S}}^{-1} F_l^{\top} \left(Q_l^{-1} + F_l \Omega_{\mathsf{S}}^{-1} F_l^{\top} \right)^{-1} F_l \Omega_{\mathsf{S}}^{-1} \right),
\end{equation}
offering a substantial reduction in runtime compared to full inversion.

\begin{algorithm}[tb]
    \small
    \caption{Fast Low-Rank Greedy}
    \label{alg::fast_greedy}
    \begin{algorithmic}[1]
        \STATE {\bfseries Input:} $\Omega_{\emptyset}$, $\{F_l, Q_l\}$ for all $l \in \mathsf{U}$, and $\kappa$
        \STATE {\bfseries Output:} subset $\mathsf{S} \subset \mathsf{U}$ with $|\mathsf{S}| = \kappa$, and $\Omega_{\mathsf{S}}^{-1}$
        \STATE $\mathsf{S} \gets \emptyset$, \quad $\Omega_{\mathsf{S}}^{-1} \gets \Omega_{\emptyset}^{-1}$
        \WHILE{$|\mathsf{S}| < \kappa$}
            \FOR{each $l \in \mathsf{U} \setminus \mathsf{S}$}
                \STATE $W_l \gets Q_l^{-1} + F_l \Omega_{\mathsf{S}}^{-1} F_l^{\top}$
                \STATE $f_l(\mathsf{S}) \gets \tr \left( \Omega_{\mathsf{S}}^{-1} F_l^{\top} W_l^{-1} F_l \Omega_{\mathsf{S}}^{-1} \right)$
            \ENDFOR
            \STATE $j \gets \arg \max_{l \in \mathsf{U} \setminus \mathsf{S}} f_l(\mathsf{S})$
            \STATE $\mathsf{S} \gets \mathsf{S} \cup \{j\}$
            \STATE $\Omega_{\mathsf{S}}^{-1} \gets \Omega_{\mathsf{S}}^{-1} - \Omega_{\mathsf{S}}^{-1} F_j^{\top} W_j^{-1} F_j \Omega_{\mathsf{S}}^{-1}$ \hfill\(\triangleright\) SMW update
        \ENDWHILE
    \end{algorithmic}
\end{algorithm}

Algorithm~\ref{alg::fast_greedy} presents the full greedy selection method based on the SMW-based update~\eqref{eq:SMW} and the marginal gain computation~\eqref{eq:marginal_info}. This version of the greedy algorithm maintains the same theoretical guarantees as the classic method but reduces each iteration’s computational cost from $\mathcal{O}(N n^3)$ to $\mathcal{O}(N n n_\ell^2)$.

Although this method significantly reduces the cubic dependence on $T$, the overall complexity still scales linearly with $\kappa$, the number of selected landmarks. To further enhance efficiency, we draw inspiration from the randomized greedy strategy introduced in~\cite{hashemi2020randomized}. This approach samples a random subset of candidates at each iteration, which reduces the evaluation cost without severely degrading performance. While it is possible to apply the low-rank SMW update within this randomized framework, we opt not to combine the two here in order to keep the conceptual contributions clearly delineated.

\subsection{Randomized Greedy}\label{sec:randomized}
The complexity associated with the greedy algorithm for feature selection can become prohibitive. To address this, we have devised a computationally efficient randomized greedy algorithm, inspired by the techniques introduced in~\cite{hashemi2020randomized,mirzasoleiman2015lazier}. This algorithm, detailed in Algorithm~\ref{alg::random}, is accompanied by a comprehensive explanation of its performance guarantees in this section.

Algorithm~\ref{alg::random} performs the task of feature selection as follows. At each iteration, a subset $\mathsf{V}$ of size $r$ is sampled uniformly at random and without replacement from the set of available features at the given time frame. The marginal gain provided by each of these $r$ features to the objective function $f(\cdot)$ in~\eqref{problem::main} is computed, and the feature yielding the highest marginal gain is added to the set of selected features. This procedure is repeated $\kappa$ times.

The parameter $\epsilon \in [e^{-\kappa}, 1]$ in Algorithm~\ref{alg::random}, which determines the size of the uniformly sampled features at each iteration, represents a predefined constant selected to achieve a desired trade-off between performance and complexity. When $\epsilon = e^{-\kappa}$, each iteration involves all non-selected features in $\mathsf{U}$, reducing Algorithm~\ref{alg::random} to the simple greedy scheme of Algorithm~\ref{alg::greedy}. Conversely, as $\epsilon$ approaches 1, the size of $\mathsf{V}$ and, consequently, the overall computational complexity decrease. However, as we will see, this reduction in complexity comes at the cost of decreased performance.

\begin{definition}\label{def:elementwise}
    The element-wise curvature $\alpha_i$ of a monotone nondecreasing function $f$ is defined as $\alpha_i \triangleq \underset{(\mathsf{S}, \mathsf{R}, l) \in \mathsf{\Pi}_i}{\max} \frac{f_l(\mathsf{R})}{f_l(\mathsf{S})}$, where $\mathsf{\Pi}_i =  \{(\mathsf{S}, \mathsf{R}, l)~|~ \mathsf{S} \subset \mathsf{R} \subset \mathsf{U}, l \in \mathsf{U} \setminus \mathsf{R}, |\mathsf{R} \setminus \mathsf{S}| = i, |\mathsf{U}| = N \}$.
    The maximum element-wise curvature is denoted by $\alpha_{\max} = \max_{i} \alpha_i$.
\end{definition}

A set function is submodular if and only if $\alpha_{\max} \leq 1$. In the following theoretical analysis, we evaluate Algorithm~\ref{alg::random} and establish a bound on its performance when used to approximate a solution to problem~\eqref{problem::main}.

\begin{theorem}\label{thrm:two}
    Let $\alpha_{\max}$ be the maximum element-wise curvature of the VIN performance measure~\eqref{problem::main}, $\mathsf{S}_{\text{rand}}$ denote the subset of features selected by Algorithm~\ref{alg::random}, and $\mathsf{S}_{\text{OPT}}$ be the optimal solution to Problem~\eqref{problem::main}. Then, on expectation, $f(\mathsf{S}_{\text{rand}})$ is a constant multiplicative factor away from $f(\mathsf{S}_{\text{OPT}})$, specifically,
    \begin{equation}\label{eq:therm2}
        \mathbb{E}[f(\mathsf{S}_{\text{rand}})] \geq \left(1 - e^{-1/c} - \frac{\epsilon^{\eta}}{c} \right) f(\mathsf{S}_{\text{OPT}}),
    \end{equation}
    where $c = \max\{\alpha_{\max}, 1\}$, $\epsilon \in [e^{-\kappa}, 1]$, and $\eta = 1 + \max\big\{0, \frac{r}{2N} - \frac{1}{2(N-r)}\big\}$.
\end{theorem}

\begin{proof}
    The proof is a straightforward variation of the proof of~\cite[Theorem 2]{hashemi2020randomized} and is not repeated here.
\end{proof}

Theorem~\ref{thrm:two} demonstrates that Algorithm~\ref{alg::random} identifies a subset of visual features that, on average, achieves a VIN objective within a constant multiplicative factor of the objective obtained by the optimal set.

\begin{algorithm}[tb]
    \small
    \caption{Randomized Greedy Algorithm}
    \label{alg::random}
    \begin{algorithmic}[1]
    \STATE {\bfseries Input:} $\Omega_{\emptyset}$, $\Delta_l$ for all $l \in \mathsf{U}$, $\kappa$, and $\epsilon \in [e^{-\kappa}, 1]$
    \STATE {\bfseries Output:} subset $\mathsf{S} \subset \mathsf{U}$ with $|\mathsf{S}| = \kappa$
    \STATE $\mathsf{S} \gets \emptyset$ 
    \FOR {$j = 1$ to $\kappa$}
        \STATE sample subset $\mathsf{V}$ with $r = \frac{N}{\kappa} \log \left(\frac{1}{\epsilon}\right)$ features uniformly from $\mathsf{U} \setminus \mathsf{S}$
        \STATE $\mathsf{S} \gets \mathsf{S} \cup \left\{\arg \max_{l \in \mathsf{V}} \left[ f(\mathsf{S} \cup \{l\}) - f(\mathsf{S}) \right] \right\}$
    \ENDFOR
    \end{algorithmic}
\end{algorithm}

Drawing from the classical analysis in~\cite{nemhauser1978analysis}, we can establish that the approximation factor for the greedy algorithm is $1 - e^{-1/c}$. Thus, the term $\epsilon^{\eta}/c$ in the multiplicative factor in~\eqref{eq:therm2} quantifies the gap between the performance of the proposed randomized greedy algorithm and the traditional greedy approach of Algorithm~\ref{alg::greedy}.

Notice that the multiplicative factor in~\eqref{eq:therm2} decreases with both $c$ and $\epsilon$. While $c$ is entirely dependent on the objective, reducing $\epsilon$ increases the multiplicative factor, leading to a better approximation. However, reducing $\epsilon$ also means approaching the conventional greedy algorithm, thereby increasing complexity.

Observe that line 6 in Algorithm~\ref{alg::random} has a computational cost of $\mathcal{O}(\frac{N}{\kappa} T^3 \log(\frac{1}{\epsilon}))$, due to the requirement to evaluate $\frac{N}{\kappa} \log(\frac{1}{\epsilon})$ marginal gains, each necessitating $\mathcal{O}(T^3)$ operations. With $\kappa$ such iterations, the overall computational complexity of Algorithm~\ref{alg::random} is $\mathcal{O}(NT^3 \log(1/\epsilon))$. Here, $N$ represents the total number of extracted features at the current time frame, and $T$ denotes the forward horizon estimation. Consequently, this approach results in a significant complexity reduction, by a factor of $\kappa / \log(1/\epsilon)$. This efficiency gain is particularly beneficial in practical applications with a large number of detected features and extended forward horizons.

\subsection{Linearization-based Greedy}
This approach relies on a first-order (linear) approximation of the MSE objective~\eqref{eq:mse}.  
The approximation is valid when the Frobenius norms of the candidate information increments~\(\Delta_{l}\) (\(l\in\mathsf U\)) are sufficiently small relative to the baseline matrix~\(\Omega_{\emptyset}\).  
The next result derives the Taylor expansion of~\(\rho\) by evaluating its directional derivative.

\begin{lemma}[First-order Taylor expansion of $\rho$]
    Let $\rho : \mathbb{S}_{++}^{n} \to \mathbb{R}$ be defined by $\rho(A) \;=\; \tr\bigl(A^{-1}\bigr)$, i.e., the mean-squared-error objective in~\eqref{eq:mse}.  
    Fix $\Omega_{\emptyset} \in \mathbb{S}_{++}^{n}$ and define $\Delta_{\mathsf{S}} \;\triangleq\; \sum_{l \in \mathsf{S}} \Delta_{l}$ with $\Delta_{l} \in \mathbb{S}_{+}^{n}$.

    If $\rho$ is differentiable at $\Omega_{\emptyset}$, then for any
    sufficiently small $\epsilon > 0$,
    \begin{equation}\label{eq:Taylor}
        \rho\bigl(\Omega_{\emptyset} + \epsilon \, \Delta_{\mathsf{S}}\bigr) \;=\; \rho\bigl(\Omega_{\emptyset}\bigr) \;+\; \epsilon\,\tr\bigl(\nabla \rho(\Omega_{\emptyset}) \, \Delta_{\mathsf{S}}\bigr) \;+\; \mathcal{O}\bigl(\epsilon^{2}\bigr).
    \end{equation}
\end{lemma}
\begin{proof}
    The differentiability of $\rho$ at $\Omega_{\emptyset}$ (established later) allows us to write its first‐order Taylor expansion in the direction $\Delta_{\mathsf{S}}$.  For sufficiently small $\epsilon>0$,
    \begin{equation*}
       \rho\bigl(\Omega_{\emptyset} + \epsilon \Delta_{\mathsf{S}}\bigr) \;=\; \rho \bigl(\Omega_{\emptyset}\bigr) \;+\;
        \epsilon\,D\rho(\Omega_{\emptyset})\,[\Delta_{\mathsf{S}}] \;+\; \mathcal{O}\bigl(\epsilon^{2}\bigr),     
    \end{equation*}
    where $D\rho(\Omega_{\emptyset})[\cdot]$ denotes the directional derivative of $\rho$ at $\Omega_{\emptyset}$.

    On the space $\mathbb{S}^{n}$ of symmetric matrices equipped with the trace inner product $\langle X, Y \rangle \triangleq \tr(XY)$, this directional derivative can be expressed through the gradient matrix $\nabla \rho(\Omega_{\emptyset})$ as
    \begin{equation*}
        D\rho(\Omega_{\emptyset})\,[\Delta_{\mathsf{S}}] \;=\;
        \bigl\langle \nabla \rho(\Omega_{\emptyset}), \Delta_{\mathsf{S}} \bigr\rangle \;=\; \tr\bigl(\nabla \rho(\Omega_{\emptyset})\,\Delta_{\mathsf{S}}\bigr).
    \end{equation*}
    Substituting this representation into the expansion yields~\eqref{eq:Taylor}, completing the proof.
\end{proof}

Given~\eqref{eq:Taylor} and the monotonicity of the performance measure $\rho$, if the norms of the information matrices associated with the candidate features are sufficiently small, then the original problem~\eqref{problem::main} can be approximated by
\begin{equation}\label{problem::main_l}
    \max_{\mathsf{S} \subset \mathsf{U}} \; -\sum_{l \in \mathsf{S}} \tr\bigl(\nabla \rho(\Omega_{\emptyset}) \, \Delta_l\bigr)
    \quad \text{subject to} \quad |\mathsf{S}| \leq \kappa.
\end{equation}

We now show that the MSE performance measure $\rho$ is differentiable and derive its directional derivative at $\Omega_{\emptyset}$ in the direction of $\Delta_l$.
\begin{algorithm}[tb]
    \small
    \caption{Linearization-based Algorithm}
    \label{alg::linear}
    \begin{algorithmic}[1]
    \STATE {\bfseries Input:} $\Omega_{\emptyset}$, $\Delta_l$ for all $l \in \mathsf{U}$, and $\kappa$
    \STATE {\bfseries Output:} subset $\mathsf{S} \subset \mathsf{U}$ with $|\mathsf{S}| = \kappa$
    \STATE $\mathsf{S} \gets \emptyset$
    \STATE compute $\tr (\Omega_{\emptyset}^{-2} \Delta_l)$ for all $l \in \mathsf{U}$
    \STATE select the indices of the $\kappa$ largest elements to form $\mathsf{S}$
    \end{algorithmic}
\end{algorithm}

Let~\(\operatorname{sort} : \mathbb{R}^{n} \to \mathbb{R}^{n}\) denote the operator that rearranges a vector’s components in non-increasing order.  
A function~\(f : \mathbb{R}^{n} \to \mathbb{R}\) is called \emph{symmetric} if
\(f(x) = f\bigl(\operatorname{sort}(x)\bigr)\) for every~\(x \in \mathbb{R}^{n}\); that is, permuting the entries of~\(x\) leaves the value of~\(f\) unchanged.  
When a symmetric function is applied to the eigenvalues of a symmetric matrix, the resulting mapping is referred to as a \emph{spectral function}.

The MSE-based objective $\rho : \mathbb{S}_{++}^{n} \to \mathbb{R}_{>0}$, $\rho(\Omega) \;=\; \tr \bigl(\Omega^{-1}\bigr)$ with $n = 9T + 9$ is a spectral function whose value depends solely on the eigenvalues of~\(\Omega\). In particular, we can write~\(\rho\) as the composition of the eigenvalue map~\(\lambda\) with the scalar function~\(\phi\): $\rho(\Omega) = (\phi\circ\lambda) (\Omega) = \phi\bigl(\lambda(\Omega)\bigr)$, with $\Omega\in\mathbb{S}_{++}^{n}$. Here $\lambda(\Omega)\,=\, \bigl[\lambda_{1}(\Omega),\,\lambda_{2}(\Omega),\ldots,\lambda_{n}(\Omega)\bigr]^{\mathsf T}$ is a vector in $\mathbb{R}_{>0}^{n}$ that collects the eigenvalues of~\(\Omega\), and $\phi(\lambda)\;=\;\sum_{i=1}^{n}\frac{1}{\lambda_{i}}$ with $\lambda\in\mathbb{R}_{>0}^{n}$.
The next lemma is a restatement of~\cite[Corollary~5.2.5]{borwein2006convex}.

\begin{lemma}[Spectral differentiability]\label{lem:spec_diff}
    Let~\(\phi : \mathbb{R}^{n} \to \mathbb{R}\) be symmetric, closed, and convex.  
    Then~\(\phi \circ \lambda\) is differentiable at a matrix~\(X \in \mathbb{S}_{++}^{n}\) if and only if~\(\phi\) is differentiable at~\(\lambda(X)\).
\end{lemma}

\begin{proposition}\label{prop:differentiable}
    The MSE objective $\rho(\Omega) \,=\, \tr\bigl(\Omega^{-1}\bigr) \,=\, \sum_{i=1}^{n} \frac{1}{\lambda_{i}(\Omega)}$, is differentiable at every~\(\Omega \in \mathbb{S}_{++}^{n}\).
\end{proposition}

In the reminder of this section, we provide a closed-form expression for the directional derivative of the MSE objective~\(\rho\). We also establish an explicit bound on the quadratic remainder term in its Taylor expansion, thereby identifying the regime in which the linear approximation remains accurate.

\begin{proposition}[Quadratic‐order Error Bound]\label{prop:quadratic_error}
Let $\Omega_{\emptyset}\in\mathbb{S}_{++}^{n}$ be the base information matrix, and define $\Delta_{\mathsf{S}} \,=\, \sum_{\,l\in\mathsf S} \Delta_{l}$ for any $\mathsf S \subseteq \mathsf U$ with $|\mathsf S|\le \kappa$. 
Assume each $\Delta_{l}\succeq 0$ satisfies $\|\Delta_{l}\|_{F}\,\le\,\zeta$, and choose $\epsilon>0$ so that ~\(\epsilon\,\|\Omega_{\emptyset}^{-1}\,\Delta_{\mathsf S}\|_{2} < 1\). Then
\begin{align}\label{eq:quadratic_expansion}
  \rho\bigl(\Omega_{\emptyset} + \epsilon\,\Delta_{\mathsf S}\bigr)
  \;& = \;
  \rho(\Omega_{\emptyset})
  \;-\;\epsilon \sum_{\,l\in\mathsf S} r_{l}\bigl(\Omega_{\emptyset}^{2}\bigr) \nonumber \\
  & +\; 
  \underbrace{
    \epsilon^{2}\,\tr\Bigl(
      \Omega_{\emptyset}^{-1}\,\Delta_{\mathsf S}\,\Omega_{\emptyset}^{-1}\,\Delta_{\mathsf S}\,\Omega_{\emptyset}^{-1}
    \Bigr)
  }_{\triangleq R_{\mathsf S}(\epsilon)} 
  \;+\;\mathcal{O}(\epsilon^{3})
\end{align}
where $r_{l}\bigl(\Omega_{\emptyset}^{2}\bigr) \,=\, \tr\bigl(\Omega_{\emptyset}^{-2}\,\Delta_{l}\bigr)$, is the leverage score of the $l$-th feature with respect to $\Omega_{\emptyset}^2$ (see~\cite{mousavi2020estimation} and the references therein for a complete definition of leverage score).  Moreover, the quadratic remainder 
$R_{\mathsf S}(\epsilon)$ is nonnegative and satisfies
{\small\begin{equation}\label{eq:quadratic_bound}
  0 \;\le\; R_{\mathsf S}(\epsilon)
  \;\le\;
  \epsilon^{2}\,\|\Omega_{\emptyset}^{-1}\|_{2}^{3}\,\|\Delta_{\mathsf S}\|_{F}^{2}
  \;\le\;
  \epsilon^{2}\,\|\Omega_{\emptyset}^{-1}\|_{2}^{3}\,\bigl(\kappa\,\zeta\bigr)^{2}.
\end{equation}}
\end{proposition}

{\it Implications for feature selection.}
Proposition~\ref{prop:quadratic_error} shows that the leading-order term in the Taylor expansion of the MSE objective is the (negated) sum of leverage scores, $ -\,\epsilon \sum_{l\in\mathsf S} r_{l}\bigl(\Omega_{\emptyset}^{2}\bigr)$, which is precisely the modular objective maximised in Algorithm~\ref{alg::linear}.  
Because modular objectives admit an optimal greedy solution, Algorithm~\ref{alg::linear} returns the \emph{exact} maximiser of the linear surrogate~\eqref{problem::main_l}.  
Using a heap to keep the current top~\(\kappa\) scores, the procedure runs in 
\(\mathcal{O}\bigl(N\,T^{3}\log \kappa\bigr)\) time, a reduction of roughly~\(\kappa/\log\kappa\) compared with the classical greedy algorithm.

{\it Accuracy of the linear surrogate.}
The quadratic remainder $R_{\mathsf S}(\epsilon) \;=\; \epsilon^{2}\, \tr \Bigl( \Omega_{\emptyset}^{-1}\Delta_{\mathsf S}\, \Omega_{\emptyset}^{-1}\Delta_{\mathsf S}\, \Omega_{\emptyset}^{-1} \Bigr)$, is always non-negative and satisfies the bound $0 \,\le\, R_{\mathsf S}(\epsilon) \,\le\, \epsilon^{2}\, \|\Omega_{\emptyset}^{-1}\|_{2}^{3}\, \|\Delta_{\mathsf S}\|_{F}^{2} \,\le\, \epsilon^{2}\, \|\Omega_{\emptyset}^{-1}\|_{2}^{3}\, (\kappa\zeta)^{2}$. Hence, when each individual information increment obeys $\|\Delta_{l}\|_{F}\le\zeta\ll\lambda_{\min}\bigl(\Omega_{\emptyset}\bigr)$ and at most~\(\kappa\) features are chosen, the total perturbation remains small ($\|\Delta_{\mathsf S}\|_{F}\le\kappa\zeta$).  
In this \emph{small-increment regime} the second-order term is~\(\mathcal{O}\bigl((\kappa\zeta)^{2}\bigr)\) and is dominated by the linear term for sufficiently small~\(\epsilon\).  

We adopt the small–increment condition~\(\zeta \ll \lambda_{\min}(\Omega_{\emptyset})\) for three reasons: (i) it guarantees~\(\epsilon\|\Omega_{\emptyset}^{-1}\Delta_{\mathsf S}\|_{2}<1\), so the Neumann expansion in Proposition~\ref{prop:quadratic_error} is valid; (ii) it bounds the quadratic remainder by~\(R_{\mathsf S}(\epsilon)\le\epsilon^{2}\|\Omega_{\emptyset}^{-1}\|_{2}^{3}(\kappa\zeta)^{2}=\epsilon^{2}(\kappa\zeta)^{2}\lambda_{\min}^{-3}(\Omega_{\emptyset})\ll\epsilon^{2}(\kappa\zeta)\lambda_{\min}^{-2}(\Omega_{\emptyset})\), ensuring this term is negligible compared with the linear one that scales as~\(\mathcal{O}(\kappa\zeta)\); and (iii) it keeps~\(\Omega_{\emptyset}+\epsilon\Delta_{\mathsf S}\) well conditioned and positive definite, so~\(\Omega_{\emptyset}^{-1}\) remains a stable reference for leverage scores. These points collectively justify the accuracy of the linearised surrogate in the small–increment regime.

{\it Guarantee.} Consequently, the greedy subset returned by Algorithm~\ref{alg::linear} incurs a true MSE cost that differs from the optimal by at most $\epsilon^{2}\, \|\Omega_{\emptyset}^{-1}\|_{2}^{3}\, (\kappa\zeta)^{2}$, a quantity that can be made arbitrarily small by choosing~\(\epsilon\) proportional to the typical size of the increments~\(\Delta_{l}\).  
Proposition~\ref{prop:quadratic_error} thus provides a rigorous justification for replacing the original non-submodular objective with its linearised surrogate and for employing Algorithm~\ref{alg::linear} as an efficient, near-exact feature-selection strategy.

\section{Experiments}\label{sec:experiments}
The attention pipeline discussed in this paper introduces three main challenges for real-world deployment. First, the system must integrate a feature extraction module to detect salient features from each image frame. Second, it must incorporate real onboard IMU measurements. Third, it must perform visibility checks over a future time horizon, which requires predicting landmark observability based on the robot's dynamical model and control inputs.

To evaluate the proposed selection methods under these challenges, we design two complementary sets of experiments. The first set uses sequences from the EuRoC dataset~\cite{burri2016euroc} and compares our methods against baseline approaches. This dataset provides feature measurements and IMU data, enabling a fair benchmarking against existing methods in the literature. However, since control inputs are not available in EuRoC, the visibility check is performed using ground-truth future poses rather than predicted ones based on the robot dynamical model.

To address this limitation and evaluate the full attention pipeline proposed in the paper, we conduct a second set of experiments using a 1/10-scale RC car equipped with onboard stereo cameras and an IMU. In this setup, control inputs are available, allowing us to implement visibility checks based on predicted future horizons using the robot's dynamics and control commands. Various selection strategies are tested on this platform to assess their practical performance.

The remainder of this section details the experimental setup and evaluation results for both datasets.

\begin{figure*}[tb]
    \centering
    \includegraphics[width=1.00\linewidth]{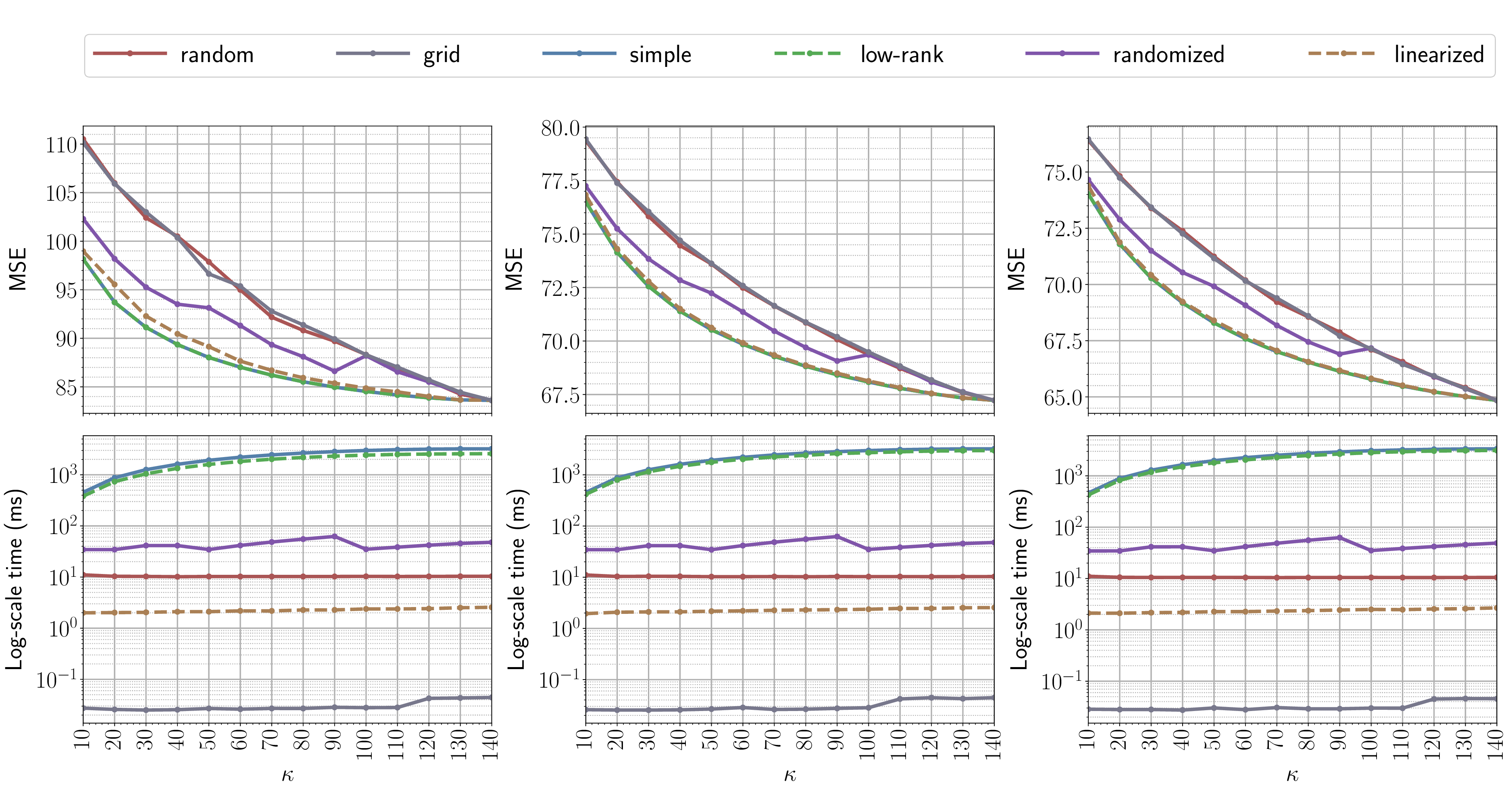}
    \caption{Performance comparison of different feature selection methods on the MH\_01\_easy sequence from the EuRoC dataset. Each column corresponds to a single randomly selected video frame. The top plot in each column shows the MSE values versus the number of selected features~\(\kappa\), while the lower plot presents the computation time on a logarithmic scale to enhance visibility and highlight discrepancies. Methods compared include uniform random selection (``random''), grid-based selection (``grid''), simple greedy (``simple''), fast low-rank greedy (``low-rank''), randomized greedy (``randomized''), and linearization-based greedy (``linearized''). For randomized methods, each experiment was repeated 20 times, and the mean values are reported. The prediction horizon $T$ is set to 13 for the information-aware selection methods, and the hyperparameter $\epsilon$ in the randomized greedy algorithm is set to 0.5. Note that MSE values and feature counts vary across frames, so results are presented for three representative frames without averaging across the sequence.}
    \label{fig:UAV_comparison}
\end{figure*}

\subsection{Benchmark Evaluation and Comparison}
For this study, we use multiple sequences from the EuRoC benchmark dataset~\cite{burri2016euroc}. The visual–inertial state estimation pipeline is based on VINS-Mono~\cite{qin2018vins}, excluding loop closure (i.e., using full odometry), which provides both the front-end and back-end processing modules. We modify the front-end to incorporate the different feature selection strategies proposed in this paper. The EuRoC benchmarks were executed on a laptop equipped with an Intel Core i9-12900H CPU (20 cores, 5.0 GHz), 32~GB of RAM, and an NVIDIA GeForce RTX 3080 Ti GPU.

The EuRoC dataset was collected using an AscTec Firefly hex-rotor UAV equipped with a stereo visual–inertial sensor suite. The stereo camera captures images at a resolution of $752 \times 480$ pixels and a frame rate of 20~Hz, while the IMU measurements are recorded at 200~Hz. In our experiments, the measurement noise parameters, along with the intrinsic and extrinsic calibrations, are set to exactly match those provided in the dataset.

In the front-end, we employ OpenCV’s implementation of the Shi–Tomasi method~\cite{shi1994good} for feature detection and the Lucas–Kanade method~\cite{lucas1981iterative} for feature tracking. The feature extractor detects up to $N = 150$ features per frame. From this set, our selection algorithms retain at most $\kappa \ll N$ features, which are then passed to the back-end for state estimation.

In practice, future pose estimates along the prediction horizon can be obtained by integrating the vehicle's dynamics using control inputs. However, since the EuRoC dataset does not provide control inputs, we approximate future poses by applying ground-truth motion increments to the current pose estimate. The feature selector operates over a prediction horizon of length~\(T\), set to~\(T = 13\) in all plots in this section. This value reflects the longest horizon that allows efficient simulation without incurring significant computational overhead. In general, larger values of~\(T\) increase the computational cost for all selection methods. Later in this section, we provide further analysis and plots illustrating the impact of~\(T\) on both performance and runtime.

We compare various VIN approaches designed to minimize the MSE and alternative objectives over the forward prediction horizon. When the algorithm is limited by a selection budget~\(\kappa\), and~\(r\) features are successfully tracked from the previous frame, only~\(\kappa - r\) new features are selected in the current frame. However, this preselection mechanism is disabled for per-frame experiments, such as the one shown in Fig.~\ref{fig:UAV_comparison}. For completeness, we initially considered including the lazy greedy algorithm (Algorithm 1 of~\cite{carlone2018attention}), known for its computational efficiency. However, it is excluded from our evaluation because it requires a lower bound (in the case of minimization) on the cost function~\( f_l(\mathsf{S}) \) for each candidate feature~\( l \) at every iteration. Existing bounds in the literature for MSE-based objectives are known to be loose, causing the lazy greedy algorithm to converge to the same performance as standard greedy, thus offering no practical advantage in this setting.

We evaluate four greedy algorithms proposed in this paper: simple greedy (Algorithm~\ref{alg::greedy}), fast low-rank greedy (Algorithm~\ref{alg::fast_greedy}), randomized greedy (Algorithm~\ref{alg::random}), and linearization-based greedy (Algorithm~\ref{alg::linear}), which are labeled in the figures as ``simple", ``low-rank", ``randomized", and ``linearized", respectively. For comparison, we also include two baseline methods. The first is a purely random selection algorithm, labeled ``random", which uniformly samples features without regard to spatial or informational structure. The second, labeled ``grid", divides each frame into a $15 \times 12$ grid and selects an equal number of features from each cell, promoting spatially uniform coverage across the image. For easier interpretation, we use the same color for each method consistently across all plots in this section.

\begin{figure}[tp]
    \centering
    \includegraphics[width=.80 \columnwidth]{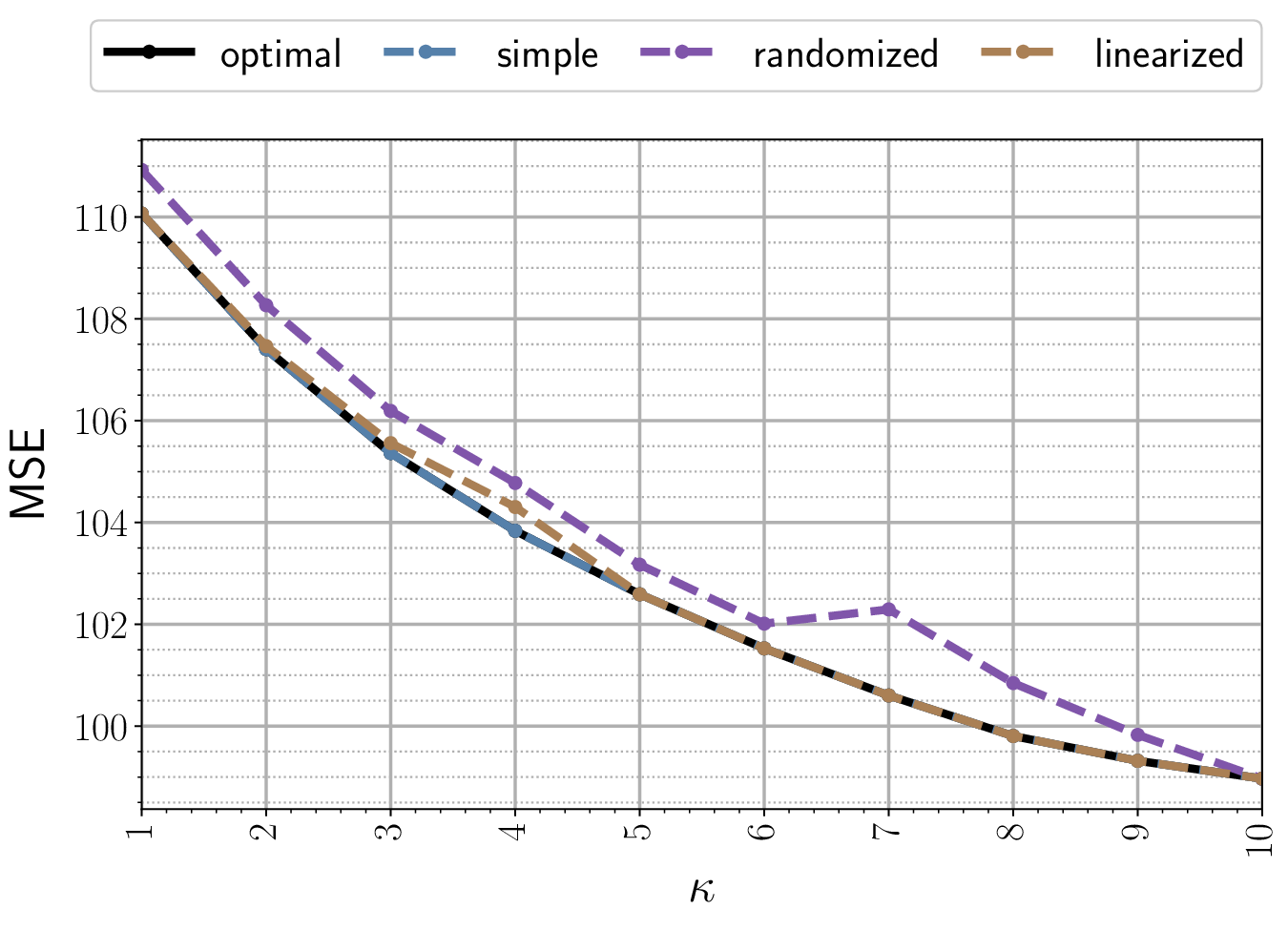}
    \caption{Performance comparison of the proposed feature selection methods on the MH\_01\_easy sequence from the EuRoC dataset. The evaluation considers selecting~\(\kappa = 1, \ldots, 10\) features from a pool of 10 candidates, allowing comparison against the optimal solution obtained via exhaustive search. The results show that the greedy algorithm achieves identical performance to the optimal method, despite the MSE-based objective not being submodular. The linearized method closely matches this performance with minimal deviation, while the randomized approach exhibits slight fluctuations due to the fixed~\(\epsilon\), yet follows a similar overall trend.}
    \label{fig:UAV_optimal_and_bound}
\end{figure}

Fig.~\ref{fig:UAV_comparison} presents the performance of the six selection methods on the MH\_01\_easy sequence from the EuRoC dataset. Each column corresponds to a single video frame. In each column, the top subplot shows the MSE values (vertical axis), and the bottom subplot shows the computation time on a logarithmic scale to enhance visibility (vertical axis), both plotted against the number of selected features~\(\kappa\) (horizontal axis). The MSE values vary significantly across frames due to changes in the robot's location along the trajectory, making it difficult to meaningfully report average values across all frames (note the differing MSE scales in each column). Additionally, the total number of features detected by the extractor varies from frame to frame, which prevents consistent plotting of mean performance curves across methods for a fixed range of~\(\kappa\). Therefore, we report results for three randomly selected frames to provide representative comparisons without aggregating over inconsistent conditions.

In each curve shown in the figure, every point on the horizontal axis corresponds to a selection budget~\(\kappa\) enforced for the given frame.  
For methods involving randomization, each experiment is repeated 20 times, and the average performance is reported. The hyperparameter~\(\epsilon\) in the randomized greedy algorithm is set to~\(0.5\) for this figure.

\begin{figure}[t]
    \centering
        \begin{tabular}{c}
            \includegraphics[width=0.75\columnwidth]{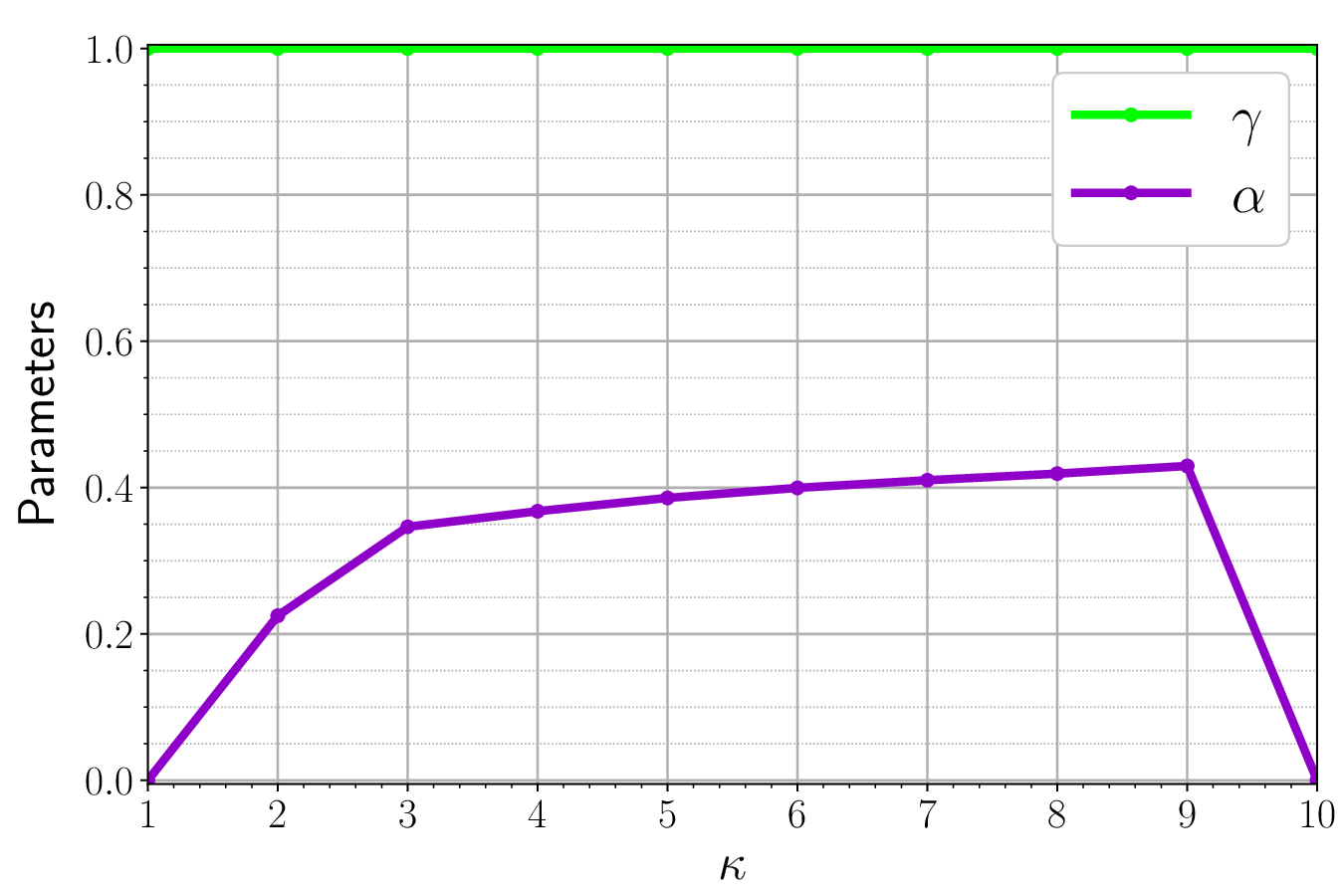}\\
            \small(a)
            \label{fig:bound-a}
        \end{tabular}
        \begin{tabular}{c}
            \includegraphics[width=0.75 \columnwidth]{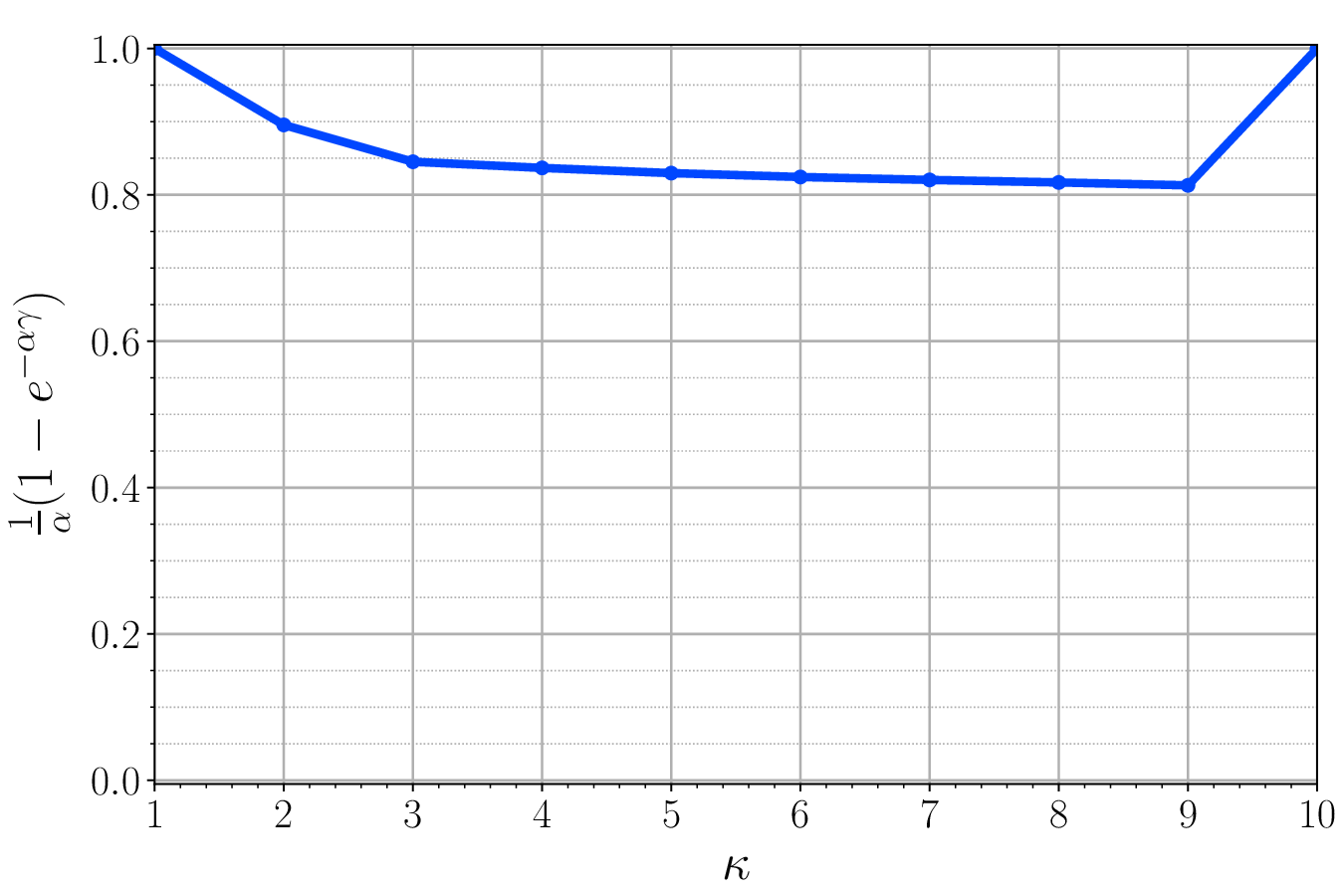}\\
            \small(b)
        \end{tabular}
        \caption{(a) Curvature ($\alpha$) and submodularity ratio ($\gamma$) as functions of $\kappa$. (b) Performance bound calculated using $\alpha$ and $\gamma$ values.}
    \label{fig:bound}%
\end{figure}

The performance of all algorithms improves as~\(\kappa\) increases. The uniform random and spatially gridded methods, while computationally efficient, underperform in terms of estimation accuracy. The randomized greedy algorithm performs better than these two, particularly at small~\(\kappa\), but still lags behind the three greedy-based strategies proposed in this paper. Although randomized greedy is relatively fast, its runtime remains higher than that of the linearization-based greedy method.

\begin{figure}[tp]
    \centering
    \includegraphics[width=.95 \columnwidth]{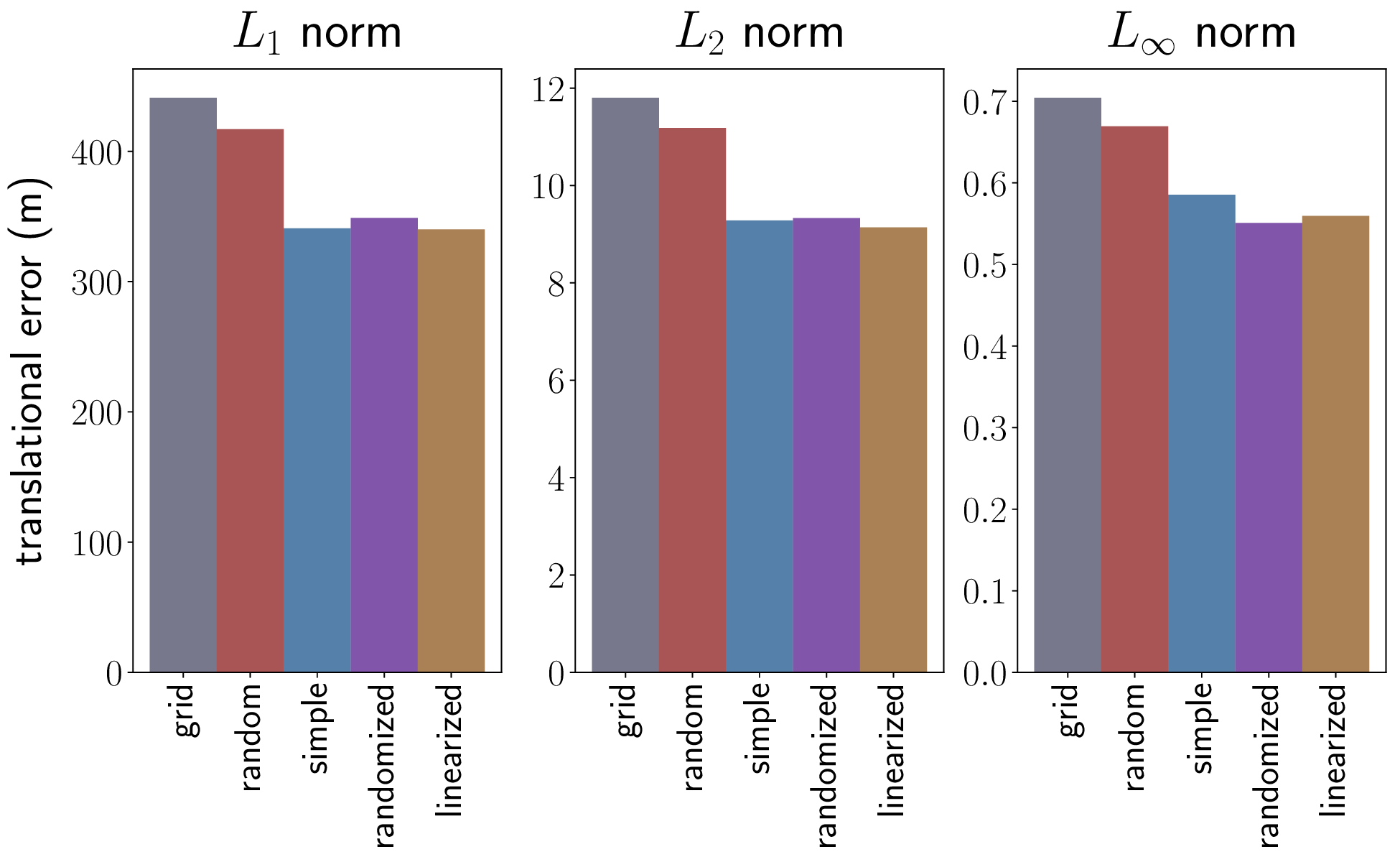}
    \caption{Translational error (in meters) between the ground-truth and estimated UAV trajectories on the MH\_01\_easy sequence from the EuRoC dataset. At each time step, the Euclidean distance between corresponding positions, computed after applying a rigid transformation for alignment, is calculated. The overall error is then summarized using the~\(L_1\),~\(L_2\), and~\(L_\infty\) norms over the sequence of these distances. The feature selection budget is fixed at~\(\kappa = 70\) for this experiment, and the selection methods correspond to those shown in Fig.~\ref{fig:UAV_comparison}.}
    \label{fig:UAV_errors}
\end{figure}

Notably, the performance of the randomized greedy method begins to converge to that of uniform random as~\(\kappa\) increases. This is a result of using a fixed~\(\epsilon\): as~\(\kappa\) grows, the number of candidates sampled in line~5 of Algorithm~\ref{alg::random} becomes a smaller fraction of the total, causing the method to behave increasingly like a random selector. Similarly, the observed fluctuations in the runtime of randomized greedy stem from this fixed~\(\epsilon\): although increasing~\(\kappa\) generally raises cost, smaller sample sizes at larger~\(\kappa\) can offset this, producing non-monotonic trends. Nevertheless, the overall time complexity remains stable and approximately constant.

As expected, the fast low-rank greedy algorithm (Algorithm~\ref{alg::fast_greedy}) closely tracks the performance of the simple greedy algorithm (Algorithm~\ref{alg::greedy}) while executing significantly faster. Interestingly, the linearization-based greedy method (Algorithm~\ref{alg::linear}) also achieves nearly identical estimation performance, but at a fraction of the computational cost, with a runtime that is effectively constant across~\(\kappa\), outperforming all other methods except for the spatially gridded baseline.

These results highlight that the choice of feature selection method should be guided by the specific requirements of the application. When computational resources are limited and fast decision-making is crucial, for instance in onboard processing on micro aerial vehicles, the linearization-based greedy algorithm emerges as the most practical option. It provides near-optimal accuracy at a fraction of the cost, with virtually constant runtime. In contrast, if achieving the lowest possible estimation error is the primary goal and computational constraints are relaxed, the fast low-rank greedy algorithm offers the best trade-off: it maintains the performance of full greedy selection while significantly reducing runtime. For offline evaluation or small-scale problems where cost is not a limiting factor, the simple greedy method remains a reliable, though expensive, benchmark. When approximate results are acceptable and further speedup is desired, randomized greedy presents a balanced middle ground that is especially effective at smaller selection budgets, though it may degrade as~\(\kappa\) grows unless~\(\epsilon\) is adaptively tuned.
Finally, in extreme low-power scenarios, such as onboard deployment with severe runtime limits, uniform random or spatially gridded methods are the cheapest alternatives, though they should be avoided if accuracy is a priority. Overall, these findings enable practitioners to make informed choices based on the performance–efficiency trade-offs relevant to their robotics pipeline.
\begin{figure}[tp]
    \centering
    \includegraphics[width=.95 \columnwidth]{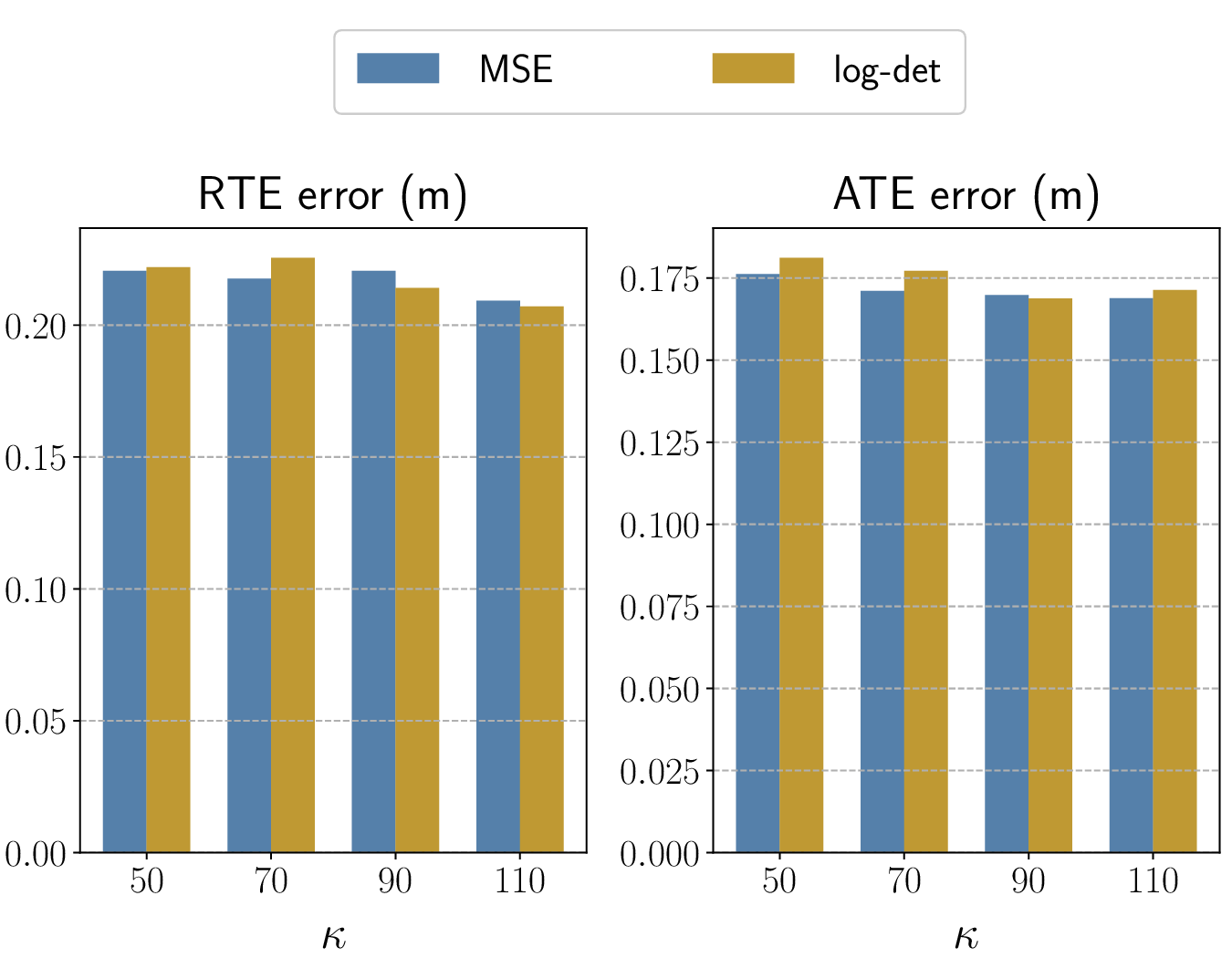}
    \caption{Comparison between greedy feature selection using MSE and log-det objectives across different values of $\kappa$. Left: Relative Translational Error (RTE). Right: Absolute Translational Error (ATE). The MH\_01\_easy sequence from the EuRoC dataset is used for evaluation.}
    \label{fig:linearized_vs_LogDet_MH_01_easy}
\end{figure}

Fig.~\ref{fig:UAV_optimal_and_bound} presents a numerical evaluation comparing various selection methods proposed in this paper against the optimal solution, thereby highlighting their relative performance bounds. To make exhaustive combinatorial search tractable, the experiment is limited to selecting from the top 10 candidate features, screened through a quality check, for a single frame of the MH\_01\_easy sequence. All other hyperparameters used in this experiment are identical to those in Fig.~\ref{fig:UAV_comparison}. 
Interestingly, the greedy and optimal results are identical. This is notable because the MSE-based VIN objective in~\eqref{eq:mse} is not generally submodular, and the greedy algorithm typically offers weaker theoretical guarantees, as characterized in~\eqref{eq:greedy_bound}. Nonetheless, this observation is consistent with findings from prior work~\cite{carlone2018attention} and related studies in other domains~\cite{zhang2015sensor,vafaee2023exploring}.

With only minor fluctuations, the linearized algorithm closely tracks the performance of both the greedy and optimal solutions. This supports our conclusion that, although linearization provides a surrogate for the MSE objective, it yields nearly identical performance while being significantly faster. As expected, the randomized greedy algorithm, despite some variability due to the fixed~\(\epsilon\), follows the same general trend as the optimal, albeit with slightly reduced accuracy.

The curvature ($\alpha$) and submodularity ratio ($\gamma$) for the greedy selection shown in Fig.~\ref{fig:UAV_optimal_and_bound} were calculated via exhaustive search using their respective definitions provided in~\eqref{eq:curve} and~\eqref{eq:submo}. These values are then utilized to obtain the performance bound suggested by Theorem~\ref{thrm:one}. The values of $\alpha$ and $\gamma$ as functions of $\kappa$ are presented in Fig.~\ref{fig:bound}(a), while the corresponding calculated performance bound using these parameters is displayed in Fig.~\ref{fig:bound}(b). 
Notably, for our non-submodular MSE-based VIN performance function~\eqref{eq:mse}, the submodularity ratio $\gamma$ for the greedy selection over this specific set remains consistently equal to one for all values of $\kappa$, indicating submodular-like behavior. However, non-submodular functions can exhibit submodular behavior in certain regions or under specific conditions. For instance, the marginal gains for some particular subsets may exhibit submodular characteristics, even if the function globally violates submodularity.

The calculated bound for this greedy feature selection shows that the performance objective of the subset obtained by the greedy algorithm exceeds 80\% of the optimal performance objective. This result perfectly explains why, in Fig.~\ref{fig:UAV_optimal_and_bound}, the objective of the greedy selection closely matches that of the optimal action.

Fig.~\ref{fig:UAV_errors} presents the~\(L_1\),~\(L_2\), and~\(L_\infty\) norms of the translational error (in meters) between the ground-truth and estimated UAV trajectories, computed over the entire flight path. To this end, the estimated and ground-truth trajectories are first aligned using the procedure described in~\cite{zhang2018tutorial}. Then, at each time step, the Euclidean distance between corresponding aligned positions is calculated. The overall error is summarized using the~\(L_1\),~\(L_2\), and~\(L_\infty\) norms over the resulting sequence of distances.  
Each bar in the figure represents one of the selection methods from Fig.~\ref{fig:UAV_comparison}; the simple greedy method includes both the standard and low-rank variants, which yield identical results.

A closer look at the three panels yields several key observations. The linearization-based greedy algorithm achieves nearly identical performance to simple greedy in terms of both cumulative error (\(L_1\)) and root-mean-square error (\(L_2\)) across the full trajectory. Interestingly, the $L_{\infty}$ plot shows that linearized greedy achieves slightly lower worst-case deviation from ground-truth compared to simple greedy. While this does not imply improved robustness in a strict sense, it suggests that the linear surrogate may better regulate error accumulation in certain trajectories. In addition, randomized greedy shows comparable performance to both simple and linearization-based greedy in~\(L_1\) and~\(L_2\), and even outperforms them in terms of~\(L_\infty\). This indicates that a modest degree of random sampling can, at times, match or even exceed the performance of more computationally intensive methods. Finally, both uniform random and grid-based baselines remain the least accurate across all error metrics, reinforcing their limited suitability for precision-sensitive applications.

These insights refine our practical guidelines: when accuracy is the top priority and computational resources allow, simple (or fast low-rank) greedy remains the most reliable choice. For scenarios with tight runtime constraints, linearization-based greedy offers an attractive balance between accuracy and efficiency. Meanwhile, randomized greedy may be favored when both computational efficiency and robustness to worst-case errors are important, as it achieves competitive performance in~\(L_1\) and~\(L_2\), and attains the lowest~\(L_\infty\) error among all methods, making it suitable for scenarios where large deviations must be minimized. Simpler heuristics such as uniform random or grid-based selection should be reserved for the most constrained applications, where runtime dominates and estimation precision is less critical.

Fig.~\ref{fig:linearized_vs_LogDet_MH_01_easy} compares the performance of greedy feature selection when using two different objective functions: MSE and log-det. This comparison highlights how the choice of objective influences the quality of the solution in practice.
The left subfigure presents the Relative Translational Error (RTE) across various values of $\kappa$ for both objectives, while the right subfigure shows the Absolute Translational Error (ATE) for the same settings. Both RTE and ATE are computed following the definitions provided in~\cite{zhang2018tutorial}. In most cases, the MSE-based objective yields lower errors, consistently outperforming the log-det approach. This is expected, as MSE directly quantifies estimation error, whereas the log-det objective serves as a surrogate for overall uncertainty, capturing the volume of the estimation covariance ellipsoid but not directly minimizing specific error metrics like MSE.

\begin{figure}[t]
    \centering
        \begin{tabular}{c}
            \includegraphics[width=0.80 \columnwidth]{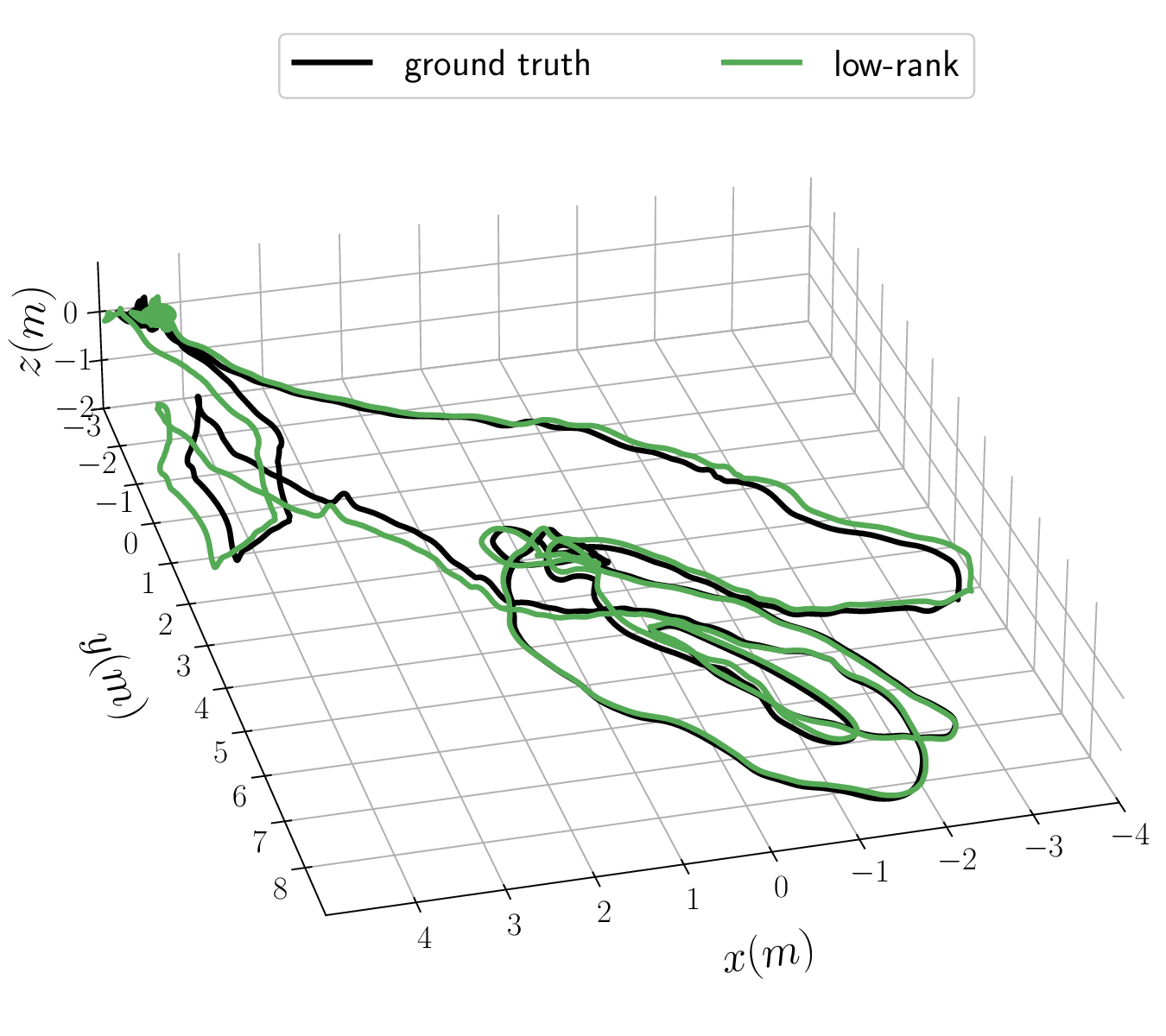}\\
            \small(a)
        \end{tabular}
        \begin{tabular}{c}
            \includegraphics[width=0.80 \columnwidth]{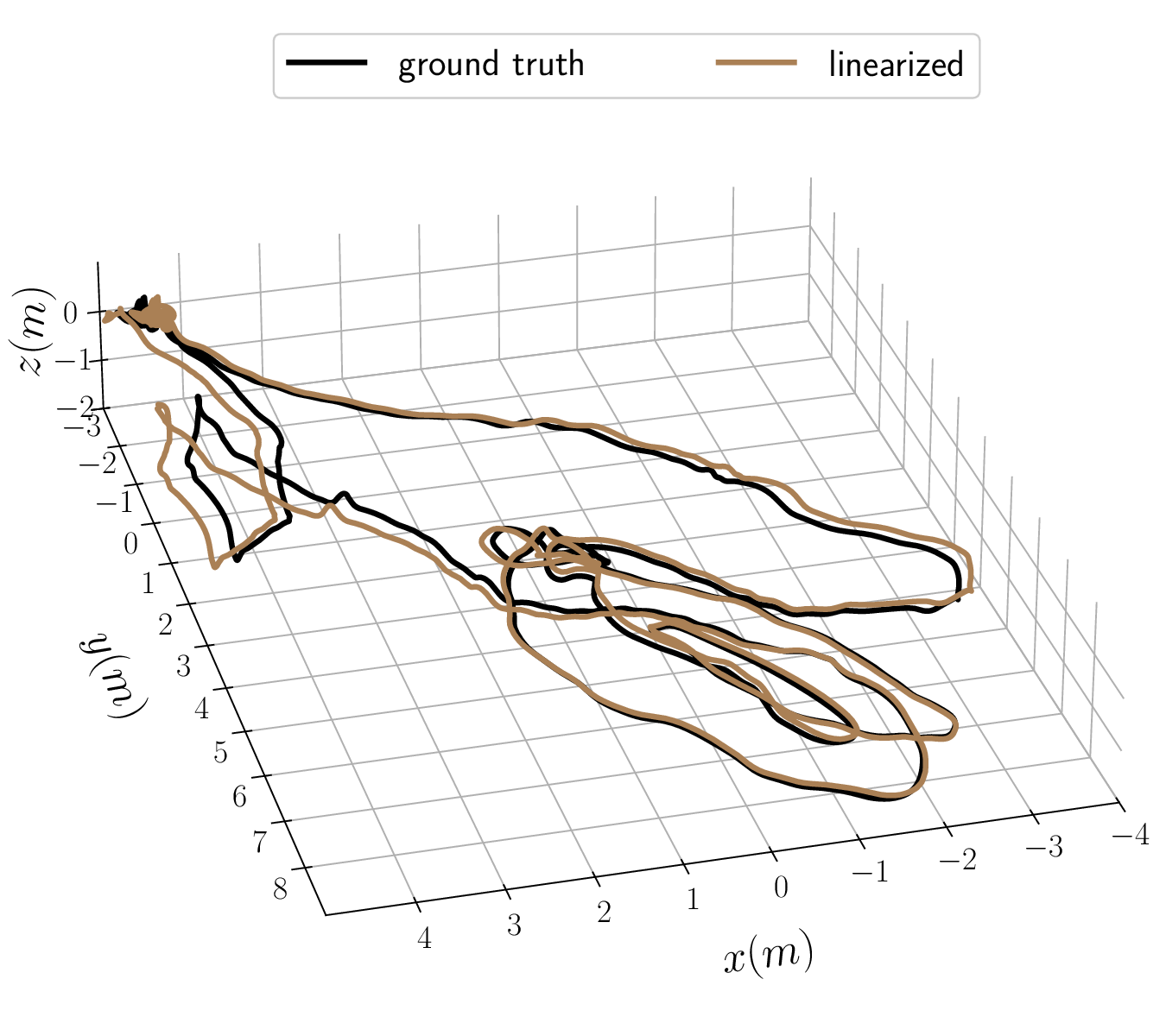}\\
            \small(b)
        \end{tabular}
        \caption {Visual comparison of the trajectories obtained using different selection techniques on the MH\_01\_easy sequence from the EuRoC dataset. The top subplot shows the estimated trajectory of the UAV using the fast low-rank greedy algorithm (in green), and the bottom subplot shows the trajectory estimated by the linearization-based greedy algorithm (Algorithm~\ref{alg::linear}), in brown. The ground-truth trajectory is shown in black in both plots for baseline comparison. Both methods follow the ground-truth trajectory closely and maintain a similar distance from it over the entire flight path, which explains their nearly identical translational error norms in Fig.~\ref{fig:UAV_errors}. For this experiment, the selection budget~\(\kappa\) is set to 70 and the prediction horizon~\(T\) is set to 13.}
    \label{fig:UAV_trajectories}%
\end{figure}

Importantly, the log-det function is known to be submodular~\cite{carlone2018attention}, which allows for the use of greedy algorithms with a worst-case approximation guarantee of $(1 - 1/e) \approx 63\%$ of the optimal. While this provides a strong theoretical guarantee, it does not always result in better empirical performance.
In contrast, the MSE objective is not submodular. However, this paper provides a constant-factor approximation guarantee for greedy selection under the MSE objective, despite its non-submodularity. As shown in Fig.~\ref{fig:bound}, our proposed approximation bound is tighter in practice, demonstrating that the greedy solution under the MSE objective performs within 80\% of the optimal solution even in the worst case. Notably, the bound developed in this paper is general and can also be applied to the log-det objective, potentially yielding tighter guarantees than the standard $(1 - 1/e)$ bound. 

To further analyze and visually assess the performance of the greedy algorithm, Fig.~\ref{fig:UAV_trajectories} presents the estimated UAV trajectories alongside the ground-truth over the entire flight path. For computational efficiency, the fast low-rank variant of greedy is used. The top subplot shows the estimated trajectory from the low-rank greedy method compared to the ground-truth, while the bottom subplot includes the trajectory obtained from the linearization-based greedy method. This method was included due to its comparable performance and significantly lower runtime, as demonstrated in earlier experiments.

The results shown in this figure support the earlier observations: both the fast low-rank greedy and linearization-based greedy methods follow the ground-truth trajectory closely, and the distance between their estimates and the ground-truth remains consistently small throughout the entire sequence. This explains why their translational error metrics (\(L_1\),~\(L_2\), and~\(L_\infty\)) in Fig.~\ref{fig:UAV_errors} are nearly identical. For this experiment, the selection budget~\(\kappa\) is set to 70.

\begin{table}[t]
    \centering
    \caption{EuRoC Results with 70 Features and MSE}
    \label{tab:euroc_50_feature_results}
    \begin{tabular}{lccc}
        \toprule
        \textbf{Sequence} & \textbf{Method} & \textbf{RTE [m]} & \textbf{ATE [m]} \\
        \midrule
        \multirow{6}{*}{MH\_01\_easy}
        & quality & $0.3843 \pm 0.076$ & $0.2323 \pm 0.114$ \\
        & random  & $0.4437 \pm 0.085$  & $0.2621 \pm 0.127$  \\
        & grid  & $0.4491 \pm 0.088$  & $0.2767 \pm 0.133$  \\
        & simple  & $0.3561 \pm 0.080$ & $0.2176 \pm 0.110$  \\
        & randomized  & $0.3613 \pm 0.064$ & $0.2187 \pm 0.105$  \\
        & linearized  & $0.3664 \pm 0.073$ & $0.2142 \pm 0.10$ \\
        \midrule
        \multirow{6}{*}{MH\_03\_medium}
        & quality  & $ 0.4753 \pm 0.197 $ & $ 0.2923 \pm 0.137 $  \\
        & random & $ 0.8184 \pm 0.284 $ & $ 0.3300 \pm 0.143 $  \\
        & grid  & $ 0.5366 \pm 0.195 $ & $ 0.3302 \pm 0.142 $  \\
        & simple  & $ 0.5392 \pm 0.203 $ & $ 0.3197 \pm 0.154 $  \\
        & randomized  & $ 0.4971 \pm 0.187 $ & $ 0.3066 \pm 0.135 $  \\
        & linearized  & $ 0.4825 \pm 0.192 $ & $ 0.2947 \pm 0.132 $  \\
        \midrule
        \multirow{6}{*}{MH\_05\_difficult}
        & quality & $ 0.7011 \pm 0.214 $ & $ 0.3150 \pm 0.075 $  \\
        & random  & $ 0.6655 \pm 0.172 $ & $ 0.3005 \pm 0.091 $  \\
        & grid  & $ 0.6482 \pm 0.166 $ & $ 0.3091 \pm 0.084 $  \\
        & simple  & $ 0.7204 \pm 0.220 $ & $ 0.3189 \pm 0.091 $  \\
        & randomized  & $ 0.6623 \pm 0.168 $ & $ 0.3173 \pm 0.078 $  \\
        & linearized  & $ 0.6606 \pm 0.200 $ & $ 0.3039 \pm 0.074 $  \\
        \midrule        
    \end{tabular}
    \label{table:sequences}
\end{table}
To exploit more deeply the EuRoC dataset for more meaningful results, Table~\ref{table:sequences} presents the performance of different feature selection strategies on three sequences from this benchmark, covering a range of difficulty levels. Challenging datasets such as MH\_05\_difficult involve fast motion and rapid viewpoint changes. All results are reported under a strict feature budget of~\(\kappa = 70\) features per frame. This constraint ensures fair comparison across methods and highlights the ability of each selection strategy to operate effectively under limited visual information.
\begin{figure}[tb]
    \centering
    \includegraphics[width=0.90\columnwidth]{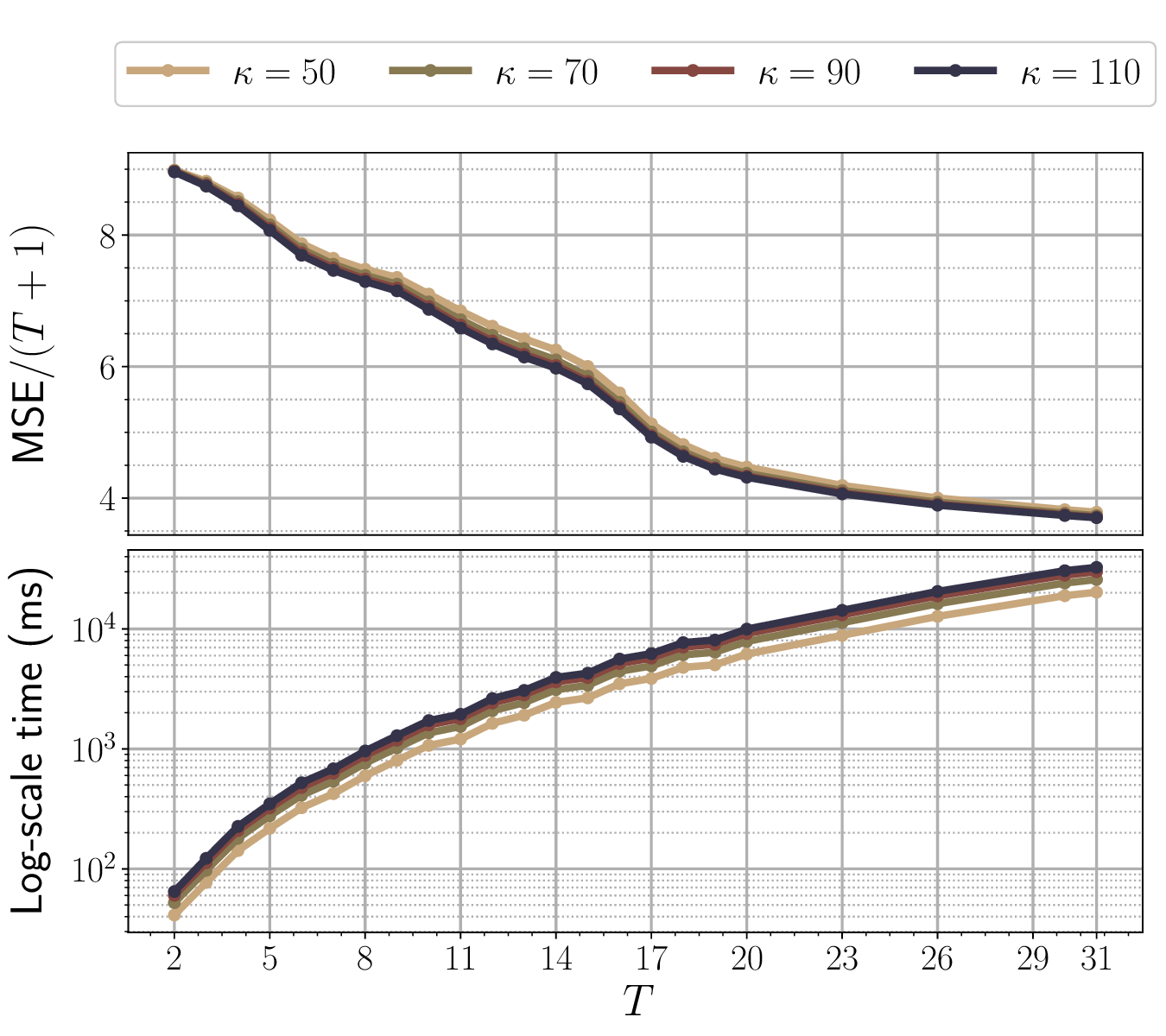}
    \caption{Effect of prediction horizon~\(T\) on estimation performance and computational cost for a single frame from the MH\_01\_easy sequence in the EuRoC dataset. The simple greedy algorithm is applied to the MSE-based objective. Top: Scaled MSE (trace of the inverse information matrix) as a function of horizon length~\(T\); performance improves with longer horizons until saturation. Bottom: Computational time (log scale) versus~\(T\); runtime increases with horizon length. Each curve represents a different feature budget~\(\kappa\).}
    \label{fig:horizon_length_analysis_MH_01_easy}
\end{figure}

We include a simple baseline method, denoted ``quality", which selects the~\(\kappa\) features with the highest Shi-Tomasi scores using the \texttt{goodFeaturesToTrack} function in OpenCV. This method is commonly used in VIN pipelines and only accounts for the appearance quality of the visual feature. Table~\ref{table:sequences} reports the Relative Translational Error (RTE) for all approaches, with the Absolute Translational Error (ATE) also included for completeness. As before, both RTE and ATE are computed following the definitions provided in~\cite{zhang2018tutorial}.

The results confirm the trends observed earlier. On the easier sequence (MH\_01\_easy), the simple greedy method achieves the lowest error, with randomized and linearization-based greedy methods performing nearly as well. The ``quality" baseline performs competitively on this sequence, which is expected given its moderate motion and stable feature tracks. However, on more dynamic sequences such as MH\_03\_medium and MH\_05\_difficult, methods that incorporate motion-aware criteria (such as greedy, randomized, and linearized) consistently outperform the quality-based, random, and grid-based baselines. 

In MH\_05\_difficult, where motion is more aggressive, the linearization-based method achieves the lowest ATE and RTE among all strategies, except for one baseline in each case. This result suggests that its surrogate objective, despite being an approximation of MSE, remains effective even in challenging settings. Randomized greedy also performs well across all sequences, striking a good balance between computational cost and accuracy, especially in medium-difficulty settings. 

These results further emphasize that appearance-based methods alone (e.g., ``quality") may suffice for simple scenes but quickly degrade in more dynamic environments. In contrast, strategies that account for motion dynamics, such as the proposed greedy-based methods, offer more consistent and robust performance across a range of operating conditions. 

Several existing feature selection approaches, such as~\cite{zhao2020good,zhang2021balancing}, employ a similar idea but are limited to a one-step prediction horizon (\(T = 1\)). To examine the effect of increasing the prediction horizon~\(T\) on both performance and computational cost, we conducted an experiment using the MH\_01\_easy sequence from the EuRoC dataset. In this experiment, we applied the simple greedy algorithm to an MSE-based objective function.

Fig.~\ref{fig:horizon_length_analysis_MH_01_easy} summarizes the results. The top plot shows the scaled MSE performance, measured as the trace of the inverse information matrix, as a function of the horizon~\(T\). Since the size of the information matrix increases with~\((T+1)\), we normalize the MSE by~\(1/(T+1)\) to enable a fair comparison across different horizon lengths. The bottom plot presents the corresponding computational time (on a logarithmic scale) versus~\(T\). Each curve corresponds to a different feature budget~\(\kappa\).

Two main observations can be made from the Fig.~\ref{fig:horizon_length_analysis_MH_01_easy}. First, for a fixed value of~\(T\), larger feature budgets~\(\kappa\) lead to better MSE performance but also higher computational cost. Second, increasing the horizon~\(T\) consistently improves performance up to a point. This improvement occurs because a longer horizon allows the algorithm to account for features that remain visible over extended periods, which tend to be more informative and reduce estimation error. However, once the horizon exceeds the visibility duration of these features, the performance gain saturates. Based on the figure, this saturation becomes evident when~\(T\) reaches around 23 or higher. Given that the EuRoC camera operates at 20 Hz, a horizon of~\(T = 23\) corresponds to approximately one second into the future.

We verified that the same qualitative scaling trends with respect to $\kappa$ and $T$ hold for all algorithmic variants (low-rank, randomized, and linearized greedy). Including every corresponding plot would require multiple subfigures beyond the page limit; therefore, we make the complete set of results available in the accompanying project repository\footnotemark[1].

The MSE in~Fig.~\ref{fig:horizon_length_analysis_MH_01_easy} is computed for a single frame based on predicted feature visibility and the corresponding information matrix. Therefore, the MSE continues to improve with increasing~\(T\) until this additional visibility information is fully exploited. In contrast, computational cost increases steadily with~\(T\), as shown in the bottom plot.

The MSE and computational cost reported in Fig.~\ref{fig:horizon_length_analysis_MH_01_easy} are computed for a single frame with the highest number of detected features in the sequence, using precomputed feature visibility and the corresponding information matrix. As the prediction horizon~\(T\) increases, the MSE improves until all available visibility information is fully exploited, while the computational cost steadily increases, as shown in the bottom plot. The purpose of this figure is to illustrate the trend in MSE and computational time as a function of the prediction horizon~\(T\), rather than to evaluate the real-time applicability of the selection method.

The next subsection details the real-time implementation of the proposed methods on a realistic QCar trajectory and presents corresponding runtime results.

\begin{figure}[tb]
    \centering
    \includegraphics[width=0.95\columnwidth]{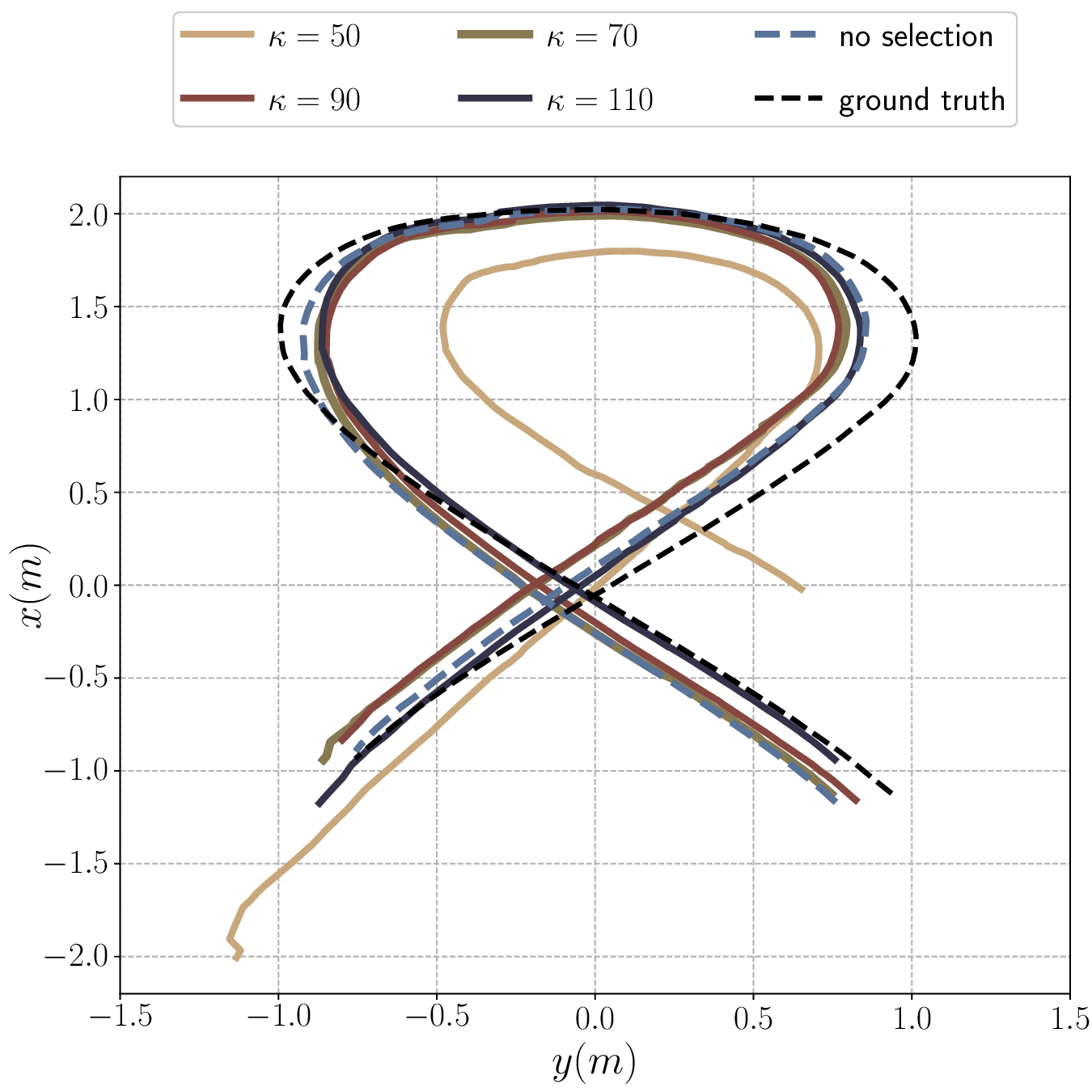}
    \caption{Comparison of estimated trajectories under different feature budgets, using MSE-based selection with the simple greedy algorithm. The ground-truth trajectory from motion capture and the trajectory obtained without feature selection (i.e., using all features) are shown for reference.}
    \label{fig:QCar_aligned_trajectories}
\end{figure}

\subsection{Experimental Case Study} \label{sec:experiments:QCar}
While the EuRoC dataset serves as a valuable benchmark for evaluating feature selection algorithms under diverse visual and motion conditions, it does not include control input data. This limitation prevents deployment of the full anticipation-based feature selection pipeline proposed in this work, which relies on forward simulation of robot dynamics for predicting future feature visibility. 

To overcome this limitation and evaluate the complete pipeline, we conducted a complementary experiment using the QCar platform, a 1/10-scale autonomous vehicle designed for academic research in robotics and self-driving technologies~\cite{QCar_Quanser}. This setup allows us to integrate control-aware future visibility prediction into our feature selection strategy.

The QCar is equipped with a 9-axis IMU and a ZED stereo camera. The IMU exhibits accelerometer and gyroscope noise standard deviations of~\(\sigma_a = 0.2\) and~\(\sigma_g = 0.2\), respectively, with corresponding bias random walks~\(\sigma_{a_b} = 0.001\) and~\(\sigma_{g_b} = 0.001\). IMU data is collected at 50~Hz, while the stereo camera captures synchronized images at a resolution of~\(1280 \times 720\) pixels at 30~Hz.

In the experiment, the QCar follows a cancer-ribbon shaped trajectory designed to introduce a range of headings. All sensor data, including IMU measurements, stereo images, and control commands, are recorded and formatted to align with the EuRoC dataset structure\footnote{\href{https://github.com/SiamiLab/NonsubmodularVisualAttention}{The GitHub repository provides access to the dataset, implementation code, and video demonstrations associated with this work.}}. Ground-truth trajectories are obtained using a network of eight motion capture cameras (OptiTrack Prime X) installed in a~\(6 \times 6~\text{m}^2\) indoor lab environment. The experimental setup, comprising the QCar, the executed trajectory, and a representative set of real-world objects, was presented earlier in Fig.~\ref{fig:diagram}.
\begin{figure}[tb]
    \centering
    \includegraphics[width=0.95\columnwidth]{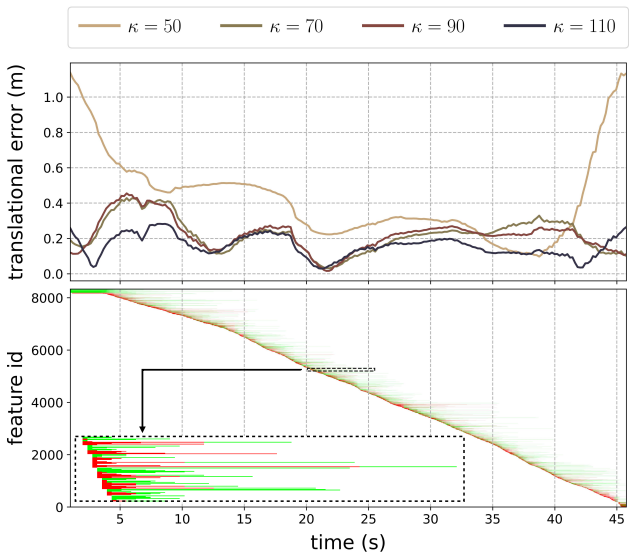}
    \caption{Top: Translational error over time between the estimated and ground-truth trajectories for different feature budgets, corresponding to Fig.~\ref{fig:QCar_aligned_trajectories}. Feature selection is performed using the simple greedy algorithm applied to the MSE-based objective. As expected, increasing the number of selected features reduces the estimation error. Bottom: Feature selection history over time for~\(\kappa = 90\). Each row represents a feature ID; green indicates selected features, red indicates observed but unselected features, and white indicates features not visible. The diagonal pattern reflects the short visibility window of each feature, and the consistent reselection of visible features demonstrates the expected behavior of the algorithm.}
    \label{fig:landmark_selection_in_time}
\end{figure}

We extended the open-source VINS-Mono framework~\cite{qin2018vins}, excluding loop closure (i.e., using full odometry), to incorporate our feature selection techniques. In the anticipation-based approach, we simulate future motion using a discrete-time kinematic bicycle model. This model propagates future states based on the current pose estimate from the VIO backend and recent control commands (linear velocity and steering angle). Specifically, at each prediction step~\(h\), the heading~\(\psi_h\) and position~\(t_h \in \mathbb{R}^3\) are updated according to:
\begin{align}\label{eq:bicycle}
    t_{h+1} &= t_h + u_h
    \begin{bmatrix} \cos\psi_h\\ \sin\psi_h\\0
    \end{bmatrix} \cdot \Delta t, \nonumber \\
    \psi_{h+1} &= \psi_h + \frac{u_h}{L} \tan\delta_h \cdot \Delta t,    
\end{align}
where~\(u_h\) and~\(\delta_h\) are the linear speed and steering angle at step~\(h\),~\(L\) is the vehicle’s wheelbase, and~\(\Delta t\) is the integration time step. The predicted trajectory~\(\{(t_h, \psi_h)\}\) allows forward visibility prediction and enables more informed feature selection. Although a constant velocity is assumed during propagation, small disturbances are naturally present in practice.

We adapted the open-source VINS-Mono framework~\cite{qin2018vins} to incorporate our feature selection methods. As in the EuRoC experiment, we use OpenCV’s Shi–Tomasi~\cite{shi1994good} for feature detection and Lucas–Kanade~\cite{lucas1981iterative} for tracking, extracting up to~\(N = 150\) features per frame. Our selection algorithms retain at most~\(\kappa \ll N\) features, which are then used by the back-end for state estimation.

\begin{figure}[tb]
    \centering
    \includegraphics[width=0.90\columnwidth]{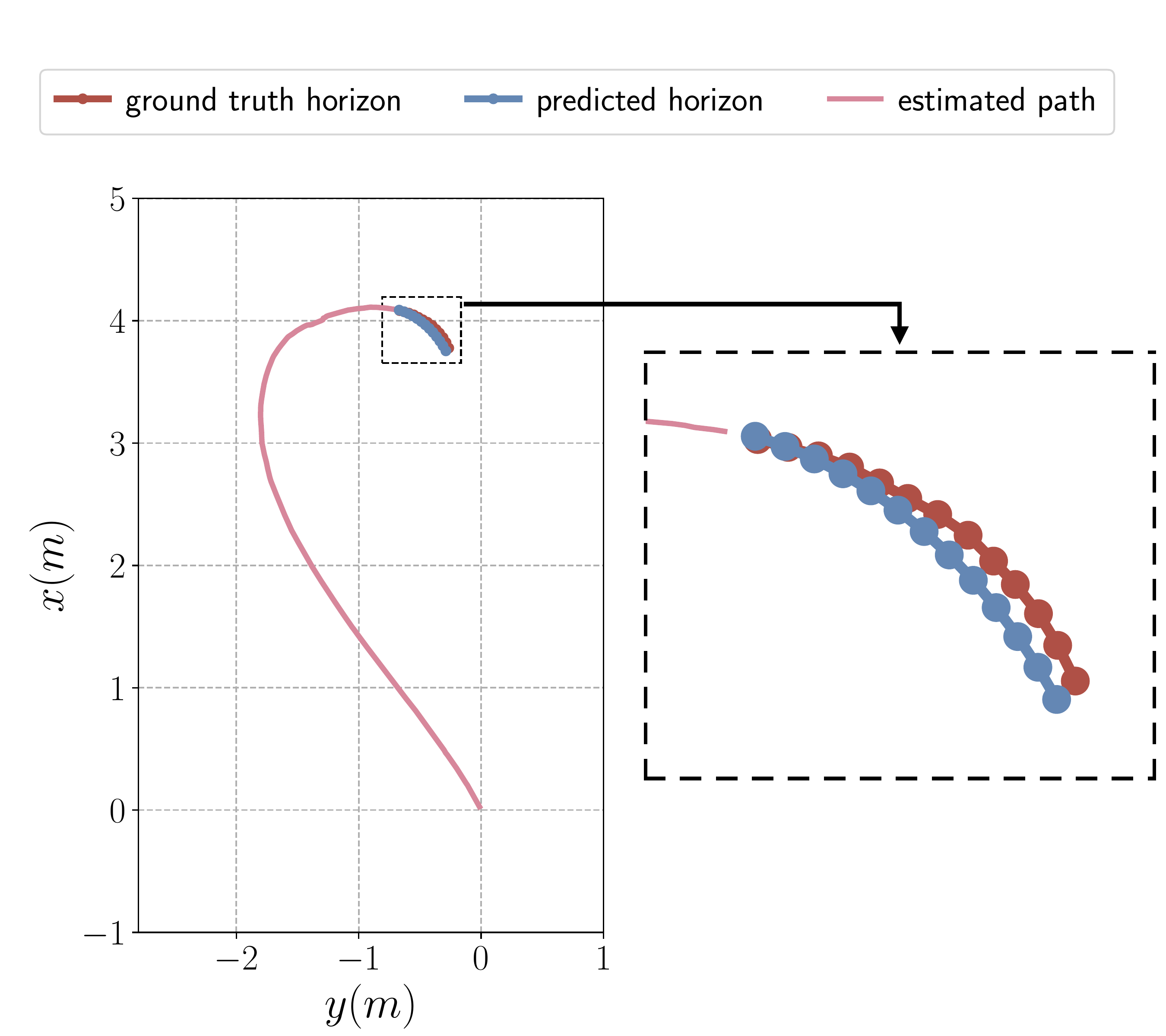}
    \caption{Comparison between the predicted future horizon generated by the kinematic bicycle model and the ground-truth trajectory. The right panel shows a zoomed-in view of the predicted horizon to highlight the alignment between the simulated and actual motion.} 
    \label{fig:horizon_generator_plot}
\end{figure}

As in the EuRoC experiments, the prediction horizon is set to~\(T = 13\).
To evaluate the effect of the feature budget on estimation accuracy, Fig.~\ref{fig:QCar_aligned_trajectories} shows estimated trajectories under different values of~\(\kappa\), the number of features selected for back-end optimization. The simple greedy algorithm is applied to our MSE-based objective for feature selection. For comparison, we also show the ground-truth trajectory from the motion capture system and the trajectory obtained without any feature selection (i.e., using all extracted features at all times), labeled as ``no selection". As expected, increasing the number of selected features leads to more accurate estimates, with trajectories that deviate less from the ground truth. Notably, the trajectory obtained under a moderate selection budget of $\kappa = 70$ achieves visual accuracy comparable to the no selection case, demonstrating the efficiency of the proposed selection method.

Quantitatively, the ATE computed for each trajectory following~\cite{zhang2018tutorial} further supports this discussion.  
The ``no selection'' case achieves \(0.1681 \pm 0.069\,\mathrm{m}\) while the greedy selection with \(\kappa\!=\!50,70,90,\) and \(110\) yields \(0.4966 \pm 0.248\), \(0.2394 \pm 0.094\), \(0.2400 \pm 0.091\), and \(0.1684 \pm 0.060\,\mathrm{m}\), respectively, confirming that budgets of \(\kappa \ge 70\) attain accuracy statistically comparable to using all features.

The top plot in Fig.~\ref{fig:landmark_selection_in_time} shows the translational error (vertical axis) between the estimated and ground-truth trajectories from Fig.~\ref{fig:QCar_aligned_trajectories}, plotted over time (horizontal axis). While the error fluctuates along the trajectory due to changes in robot motion and feature visibility, increasing the feature budget consistently reduces the overall estimation error, as expected.
The bottom plot depicts the selection history of visual features over time when selecting up to~\(\kappa = 90\) features per frame (from a maximum of 150). The vertical axis corresponds to feature IDs. White indicates the feature was not observed at that time, red indicates it was observed but not selected, and green indicates it was selected. The diagonal banding pattern reflects that features are typically visible only for short durations as the robot moves. Moreover, once a feature is selected, it tends to be consistently reselected in subsequent frames until it leaves the field of view, demonstrating the expected temporal consistency of the selection algorithm.

In addition, Fig.~\ref{fig:horizon_generator_plot} presents a snapshot of the algorithm at a specific point along the trajectory. At this moment, the predicted future positions using the bicycle model (discussed earlier) over a horizon of~\(T = 13\) are shown alongside the actual observed feature horizon. As illustrated, the predicted and actual horizons closely align at this randomly selected point. This figure serves as a demonstration of the accuracy of the motion prediction used in the anticipation-based feature selection.

Table~\ref{tab:qcar_runtime} presents the average per-frame runtime for different feature selection strategies in the QCar experiment under a fixed feature budget of~\(\kappa = 70\). The ``No Selection'' baseline uses all available features (up to 150 per frame) without any filtering, requiring no run time for selection but resulting in a significantly higher optimization time in Ceres. In contrast, all other methods apply selection before back-end optimization, reducing the number of features passed to Ceres and therefore improving its runtime. 
Notably, the linearized and randomized methods achieve the lowest total runtime, with the linearized approach being the most efficient overall. The simple and low-rank greedy methods, while more computationally intensive during selection, still significantly reduce the cost of the optimization stage.
\begin{table}[t]
    \centering
    \caption{ Runtime comparison for QCar experiment under $\kappa=70$ feature budget. Times are averaged per frame over the full trajectory.}
    \label{tab:qcar_runtime}
    \begin{tabular}{lccc}
    \toprule
    \textbf{Selector} & \textbf{Selection [ms]} & \textbf{Ceres [ms]} & \textbf{Total [ms]} \\
    \midrule
    No Selection      & --        & 24.75       & 24.75 \\
    Simple            & 82.01     & 13.14       & 95.14 \\
    Low-rank          & 72.25     & 13.85       & 86.10 \\
    Randomized        & 11.18     & 13.18       & 24.36 \\
    Linearized        & 7.75      & 13.61       & 21.36 \\
    \bottomrule
    \end{tabular}
\end{table}

It is important to note that the runtimes reported here differ from those in Figs.~\ref{fig:UAV_comparison} and~\ref{fig:horizon_length_analysis_MH_01_easy}, where the selection algorithms were evaluated intentionally on frames containing the maximum number of features (i.e., 150). In the QCar experiment, however, the actual number of candidate features per frame is often lower due to limited visibility and occlusion. Moreover, features selected in previous frames are directly carried over to the next frame if they remain visible. As a result, the number of new features that need to be selected is reduced, contributing to the lower selection times observed in this table.

In addition, the runtimes reported under the ``Selection" column include the total computation required per frame, including generating the forward prediction horizon, performing visibility checks, constructing the information matrices, and applying the selection algorithms. In contrast, the runtimes shown in Figs.~\ref{fig:UAV_comparison} and~\ref{fig:horizon_length_analysis_MH_01_easy} reflect only the time required to apply the selection algorithms, assuming all required inputs are already available. The lower computation times in the table, especially for the linearized method, highlight its suitability for real-time, on-the-fly applications.

To further support this experiment, we conducted an additional analysis to provide insight into the per-frame runtime of two contrasting selection methods: the computationally intensive yet effective \textit{simple greedy} method, and the efficient and fast \textit{linearized} method. Fig.~\ref{fig:time_analysis_over_frames} presents the per-frame runtime (vertical axis) of the two methods over a sequence of frames (horizontal axis) in the QCar experiment. As shown, the linearized method consistently achieves significantly lower runtime compared to the simple greedy method.

The observed variability in the runtime of the simple greedy method is expected. This variability arises due to fluctuations in the total number of visible features and, more importantly, the number of features that need to be newly selected at each frame, since some features are carried over from previous frames if they remain visible. These factors directly impact the selection time in greedy approaches. In contrast, the linearized method avoids iterative selection and instead computes a single score per feature based on the linear part of the objective, which is the main computational step, regardless of how many features are to be selected. It then selects the top-$\kappa$ features, making its runtime largely insensitive to the number of features selected in each frame.
These results highlight the suitability of the linearized method for real-time use. The experiment uses the same time horizon and feature budget as Table~\ref{tab:qcar_runtime}, with $T = 13$ and $\kappa = 70$.

The QCar experiments demonstrate the feasibility of deploying our full anticipation-based feature selection pipeline in a control-aware setting. Leveraging control inputs and a motion prediction model enables forward visibility checks, which inform the selection process. Experimental results validate that increasing the feature budget consistently enhances state estimation accuracy, and runtime evaluations reveal practical trade-offs among feature selection strategies. Overall, the proposed anticipation-based selection methods offer a promising mechanism for balancing feature usage, accuracy, and computational cost in visual-inertial estimation tasks, and may generalize to other multi-step estimation frameworks.
\begin{figure}[tb]
    \centering
    \includegraphics[width=0.90\columnwidth]{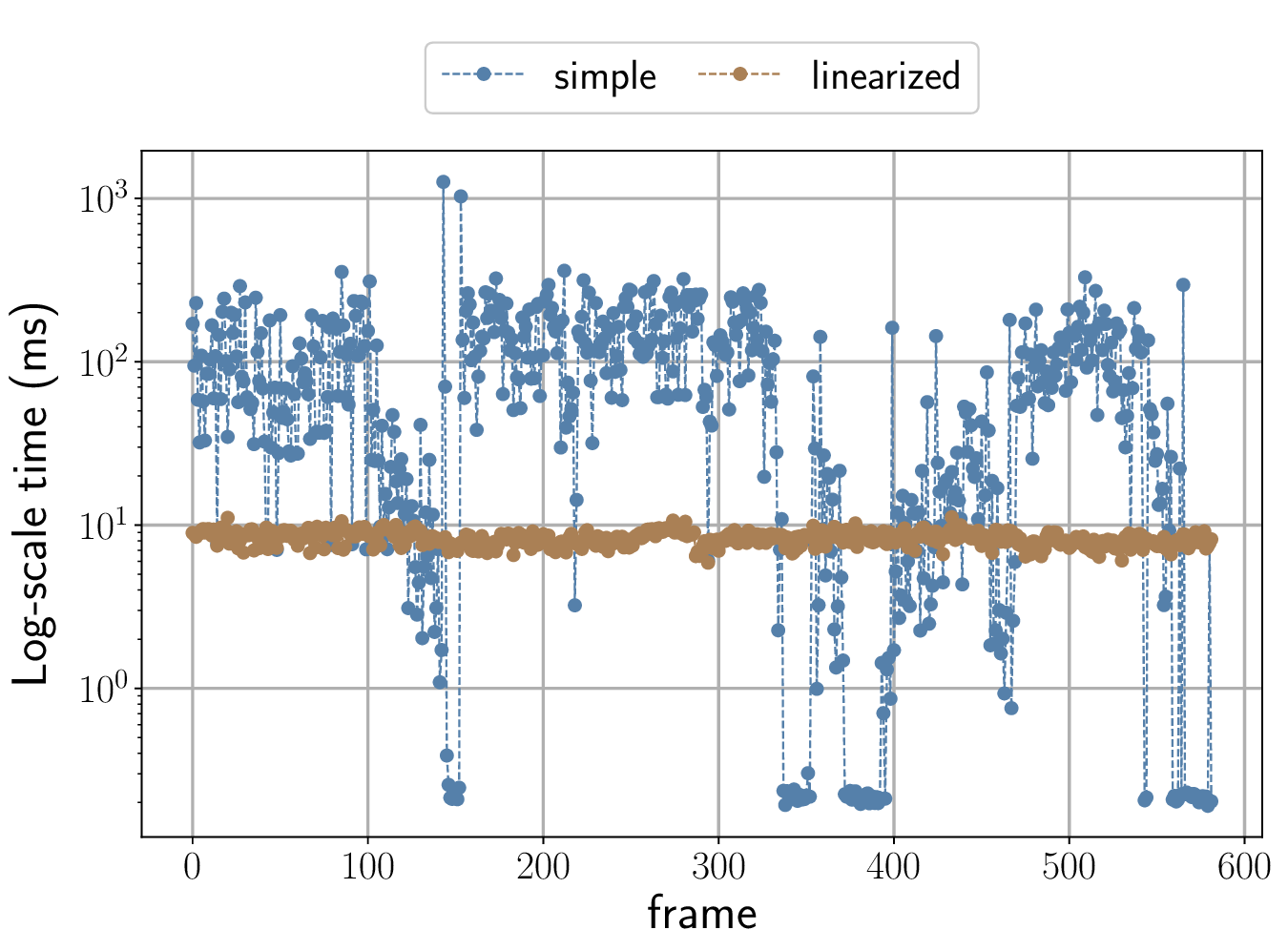}
    \caption{Per-frame runtime comparison between the simple greedy and linearized selection methods in the QCar experiment with time horizon $T = 13$ and feature budget $\kappa = 70$. The linearized method exhibits consistently lower and more stable runtimes, highlighting its suitability for real-time applications.}
    \label{fig:time_analysis_over_frames}
\end{figure}
\section{Concluding Remarks}
We have presented and rigorously evaluated four practical algorithms for task-aware feature selection in visual-inertial navigation across two experimental settings. On the EuRoC MAV benchmark, chosen for its standardized indoor sequences and wide adoption for fair comparison, we showed that our fast low-rank greedy variant matches the classic greedy method in accuracy while reducing computation time, and that our linearization-based approach operates in real time with performance close to the greedy baseline. On our custom QCar platform, designed to capture control-aware navigation dynamics, we evaluated how each method adapts under realistic motion and sensing conditions, demonstrating reliable performance without relying on idealized assumptions.

Despite these strengths, our study has two notable limitations. First, the theoretical approximation bounds based on submodularity ratio and curvature, while providing worst-case guarantees, can be conservative and may not fully predict empirical performance in all scenarios. Second, the randomized greedy variant introduces sampling variability and may require careful tuning of the randomization parameter to balance consistency and computational load.

These results confirm that, depending on application constraints, one can choose \emph{maximum accuracy} (fast low-rank greedy), \emph{extreme speed} (linearization-based greedy), or \emph{balanced performance} (randomized greedy), all with rigorous approximation guarantees. In particular, the real-time capability of the linearization-based method makes it especially well suited for on-board deployment on resource-constrained platforms without sacrificing estimation quality.

Our work underlines both the theoretical depth and the practical breadth of task-aware sensor selection, demonstrating its effectiveness from controlled benchmark scenarios to real-world robotic systems.
Looking forward, several promising directions remain:
\begin{itemize}
  \item \textit{Multi-agent extensions:} exploit cross-robot feature sharing and cooperative information gains in teams operating in dynamic environments.
  \item \textit{New sensor modalities:} generalize the modular Taylor-approximation framework to LiDAR, depth, or event-camera measurements under richer noise and motion models.
  \item \textit{Adaptive budgeting:} develop online schemes to adjust the selection budget in response to scene complexity and control objectives.
  \item \textit{Tighter theoretical bounds:} incorporate higher-order approximations or stochastic batch selection to refine performance guarantees.
    \item \textit{Matrix-free strategies:} investigate ultra-lightweight selection schemes that eliminate matrix operations entirely, enabling deployment on resource-constrained platforms requiring extreme-speed decisions.

\end{itemize}
Addressing these open questions will further enhance the applicability and impact of task-aware sensor selection across diverse robotic applications.

\appendix
\subsection{Background and Definitions}
This section presents background material and definitions that support the theoretical developments in the paper.
\begin{definition}[Closed Function]
    Let~\(f : \mathbb{R}^{n} \to \mathbb{R} \cup \{+\infty\}\) be an extended-real-valued function, and define its (effective) domain by $\operatorname{dom} f \;=\; \bigl\{x \in \mathbb{R}^{n} \mid f(x) < +\infty\bigr\}$. The function~\(f\) is said to be \emph{closed} if, for every~\(\beta \in \mathbb{R}\), the sublevel set $\bigl\{x \in \operatorname{dom} f \mid f(x) \le \beta\bigr\}$, is a closed subset of~\(\mathbb{R}^{n}\).
\end{definition}
\vspace{-1mm}
\begin{remark}\label{rem:closed}
    If~\(f\) is continuous on a closed domain, then~\(f\) is closed.  
    When~\(\operatorname{dom} f\) is open,~\(f\) is closed if and only if~\(f(x_k) \to +\infty\) for every sequence~\(\{x_k\}\) that approaches a boundary point of~\(\operatorname{dom} f\) {\cite{boyd2004convex}}.
\end{remark}
\vspace{-1mm}
\subsection{Missing Proofs}\label{sec:missing_proofs}
This section presents the proofs that complement the theoretical findings discussed in the main text of the paper.

\subsubsection{Proof of Theorem~\ref{thrm:one}}
Recall that $\tr \Omega_{\mathsf{S}}^{-1} = \sum_{i=1}^{n} \frac{1}{\lambda_i(\Omega_{\mathsf{S}})}$, where $n = 9T + 9$. Let $\Omega_{\mathsf{U}} = \Omega_{\emptyset} + \sum_{l \in \mathsf{U}} \Delta_l$ represent the information matrix considering all extracted features at the given time frame.

In the first part of the proof, we establish a lower bound for the submodularity ratio $\gamma$. Referring to Definition~\ref{def:sub-ratio}, this can be achieved by finding a lower bound for the fraction
\begin{equation}\label{eq:sub_frac}
    \frac{\sum_{l\in \mathsf{R} \setminus \mathsf{S}} f_l(\mathsf{S})}{f_{\mathsf{R}}(\mathcal{S})} = \frac{\sum_{l \in \mathsf{R} \setminus \mathsf{S}} \left[ f(\mathsf{S} \cup \{l\}) - f(\mathsf{S}) \right]}{f(\mathsf{R} \cup \mathsf{S}) - f(\mathsf{S})}.
\end{equation}

For the numerator of this fraction, we have
\begin{align}\label{eq:414}
     & \sum_{l \in \mathsf{R} \setminus \mathsf{S}} \left[f(\mathsf{S} \cup \{l\}) - f(\mathsf{S}) \right] \nonumber \\
     & = \sum_{l \in \mathsf{R} \setminus \mathsf{S}} \Big[\sum_{i=1}^{n}\frac{1}{\lambda_i(\Omega_{\mathsf{S}})} - \frac{1}{\lambda_i(\Omega_{\mathsf{S} \cup \{l\}})}\Big] \nonumber \\
     & \geq \frac{\sum_{l \in \mathsf{R} \setminus \mathsf{S}} \sum_{i=1}^{n} \lambda_i(\Omega_{\mathsf{S} \cup \{l\}}) - \lambda_i(\Omega_{\mathsf{S}})}{\lambda_{\max}^2(\Omega_{\mathsf{U}})} \nonumber \\
     & = \frac{\sum_{l \in \mathsf{R} \setminus \mathsf{S}}  \left[ \tr \, \Omega_{\mathsf{S} \cup \{l\}}  - \tr \, \Omega_{\mathsf{S}} \right]}{\lambda_{\max}^2(\Omega_{\mathsf{U}})} \nonumber \\
     & = \lambda_{\max}^{-2}(\Omega_{\mathsf{U}}) \sum_{l \in \mathsf{R} \setminus \mathsf{S}}  \tr \, \Delta_l \geq \delta \cdot | \mathsf{R} \setminus \mathsf{S}| \cdot \lambda_{\max}^{-2}(\Omega_{\mathsf{U}}),
\end{align}
where $\delta \triangleq \min_{j \in \mathsf{U}} \tr \Delta_j$.
Analogously, for the denominator, we have
\begin{align}\label{eq:428} 
    f(\mathsf{R} \cup \mathsf{S}) - f(\mathsf{S}) 
    & = \sum_{i=1}^{n}\frac{1}{\lambda_i(\Omega_{\mathsf{S}})} - \sum_{j=1}^{n} \frac{1}{\lambda_j(\Omega_{\mathsf{S} \cup \mathsf{R}})} \nonumber \\
    & \stackrel{\text{(i)}}{\leq} \sum_{i = n - |\mathsf{R} \setminus \mathsf{S}| + 1}^{n} \frac{1}{\lambda_i(\Omega_{\mathsf{S}})} - \sum_{j=1}^{|\mathsf{R} \setminus \mathsf{S}|} \frac{1}{\lambda_j(\Omega_{\mathsf{S} \cup \mathsf{R}})} \nonumber \\
    & \leq |\mathsf{R} \setminus \mathsf{S}| \left( \frac{1}{\lambda_{\min}(\Omega_{\emptyset})} - \frac{1}{\lambda_{\max}(\Omega_{\mathsf{U}})} \right) \nonumber \\
    & = |\mathsf{R} \setminus \mathsf{S}| \, \frac{\lambda_{\max}(\Omega_{\mathsf{U}}) - \lambda_{\min}(\Omega_{\emptyset})}{\lambda_{\min}(\Omega_{\emptyset}) \cdot \lambda_{\max}(\Omega_{\mathsf{U}})},
\end{align}
where inequality (i) holds due to the interlacing inequality of eigenvalues.
The combination of~\eqref{eq:414} and~\eqref{eq:428} provides the first part of the proof as
\begin{equation*}
    \gamma \geq \frac{\delta \cdot \lambda_{\min} (\Omega_{\emptyset})}{\lambda_{\max}(\Omega_{\mathsf{U}}) \cdot (\lambda_{\max}(\Omega_{\mathsf{U}}) - \lambda_{\min} (\Omega_{\emptyset}))} = \underline{\gamma}.
\end{equation*}

In the final part of the proof, we obtain the bound for the curvature. To do this, we determine the bound for $1 - \alpha$. Referring to~\eqref{eq:curve}, this task is equivalent to establishing a lower bound for
\begin{equation}\label{eq:curve_frac}
    \frac{f(\mathsf{S} \cup \mathsf{R}) - f(\mathsf{S} \setminus \{l\} \cup \mathsf{R})}{f(\mathsf{S}) - f(\mathsf{S} \setminus \{l\})}.
\end{equation}

For the numerator of this fraction, we have
\begin{align}\label{eq:456}
    & f(\mathsf{S} \cup \mathsf{R}) - f(\mathsf{S} \setminus \{l\} \cup \mathsf{R}) \nonumber \\
    & = \sum_{i = 1}^{n} \frac{1}{\lambda_i(\Omega_{\mathsf{S} \setminus \{l\} \cup \mathsf{R}})} - \sum_{j = 1}^{n} \frac{1}{\lambda_j(\Omega_{\mathsf{S} \cup \mathsf{R}})} \geq \delta \, \lambda_{\max}^{-2}(\Omega_{\mathsf{U}}),
\end{align}
where we use a similar derivation as in~\eqref{eq:414}.
Let us now consider the denominator,
\begin{equation*}
    f(\mathsf{S}) - f(\mathsf{S} \setminus \{l\}) = \sum_{i=1}^{n} \frac{1}{\lambda_i(\Omega_{\mathsf{S} \setminus \{l\}})} - \sum_{j=1}^{n} \frac{1}{\lambda_j(\Omega_{\mathsf{S}})}.
\end{equation*}
This can be upper-bounded using the \emph{Cauchy interlacing inequality} as follows:
\begin{equation}\label{eq:468}
    \frac{1}{\lambda_{\min}(\Omega_{\mathsf{S} \setminus \{l\}})} - \frac{1}{\lambda_{\max}(\Omega_{\mathsf{S}})} \leq \frac{\lambda_{\max}(\Omega_{\mathsf{U}}) - \lambda_{\min}(\Omega_{\emptyset})}{\lambda_{\min}(\Omega_{\emptyset}) \cdot \lambda_{\max}(\Omega_{\mathsf{U}})}.
\end{equation}

Combining~\eqref{eq:456} and~\eqref{eq:468} yields a lower bound on~\(1-\alpha\):
\begin{align*}
  1-\alpha
  &=\frac{f(\mathsf{S}\cup\mathsf{R})-f(\mathsf{S}\setminus\{l\}\cup\mathsf{R})}
         {f(\mathsf{S})-f(\mathsf{S}\setminus\{l\})}\\[2pt]
  &\ge
    \frac{\delta\,\lambda_{\min}\bigl(\Omega_{\emptyset}\bigr)}
         {\lambda_{\max}\bigl(\Omega_{\mathsf{U}}\bigr)\,
          \bigl(\lambda_{\max}\bigl(\Omega_{\mathsf{U}}\bigr)-\lambda_{\min}\bigl(\Omega_{\emptyset}\bigr)\bigr)}.
\end{align*}
Consequently,
\begin{equation}
  \alpha\;\le\;
  \overline{\alpha}
  \;=\;
  1-
  \frac{\delta\,\lambda_{\min}\bigl(\Omega_{\emptyset}\bigr)}
       {\lambda_{\max}\bigl(\Omega_{\mathsf{U}}\bigr)\,
        \bigl(\lambda_{\max}\bigl(\Omega_{\mathsf{U}}\bigr)-\lambda_{\min}\bigl(\Omega_{\emptyset}\bigr)\bigr)},
\end{equation}
which completes the proof.

\subsubsection{Proof of Proposition~\ref{prop:differentiable}}
Define~\(\phi(\lambda) = \sum_{i=1}^{n} 1/\lambda_i\) for all~\(\lambda \in \mathbb{R}^{n}_{>0}\).  
The function~\(\phi\) is continuous, symmetric, and convex.  
Its domain is open, and~\(\phi(\lambda) \to +\infty\) as~\(\min_{i}\lambda_i \to 0^{+}\); by Remark~\ref{rem:closed},~\(\phi\) is therefore closed.  
Because~\(\phi\) is smooth on~\(\mathbb{R}^{n}_{>0}\), it is differentiable at every~\(\lambda \in \mathbb{R}^{n}_{>0}\).  
Lemma~\ref{lem:spec_diff} then implies that~\(\rho = \phi \circ \lambda\) is differentiable at every~\(\Omega \in \mathbb{S}_{++}^{n}\).

\subsubsection{Proof of Proposition~\ref{prop:quadratic_error}}
Set $A \,\triangleq\, \Omega_{\emptyset}$, and $B \,\triangleq\, \Delta_{\mathsf S}$, so that $A\succ 0$ and $B \succeq 0$.  Under the assumption
$\epsilon\,\|A^{-1}B\|_{2} < 1$, we expand $A + \epsilon\,B \,=\, A \bigl(I + \epsilon\,A^{-1}B\bigr)$, and apply the Neumann‐series to $(I + \epsilon\,A^{-1}B)^{-1}$:
\[
\bigl(I + \epsilon\,A^{-1}B\bigr)^{-1}
  \;=\; I 
  \;-\; \epsilon\,A^{-1}B 
  \;+\; \epsilon^{2}(A^{-1}B)^{2}
  \;+\;\mathcal{O}(\epsilon^{3}).
\]
Multiplying on the right by $A^{-1}$ gives
\begin{align*}
    (A + \epsilon\,B)^{-1} \;=&\; A^{-1} \;-\; \epsilon\,A^{-1}B\,A^{-1} \nonumber \\
    & \;+\; \epsilon^{2}\,A^{-1}B\,A^{-1}B\,A^{-1} \;+\;\mathcal{O}(\epsilon^{3}).
\end{align*}
Taking the trace yields
\begin{align*}
    \tr\bigl((A +  \epsilon\,B)^{-1}&\bigr) \;=\; \tr(A^{-1}) \;-\; \epsilon\,\tr(A^{-1}BA^{-1}) \nonumber \\ &+\; \epsilon^{2}\,\tr\bigl(A^{-1}BA^{-1}BA^{-1}\bigr) \;+\;\mathcal{O}(\epsilon^{3}).
\end{align*}
Since $\rho(X)=\tr(X^{-1})$, we obtain
\begin{align*}
  \rho(A + \epsilon\,B)
  \; = & \;  \rho(A) \;-\; \epsilon\,\tr\bigl(A^{-2}B\bigr) \nonumber \\
  & +\; \epsilon^{2}\,\tr\bigl(A^{-1}BA^{-1}BA^{-1}\bigr) \;+\;\mathcal{O}(\epsilon^{3}).  
\end{align*}
Notice that
\begin{align*}
    \tr\bigl(A^{-2}B\bigr) & = \tr\Bigl(\Omega_{\emptyset}^{-2}\sum_{\,l\in\mathsf S}\Delta_{l}\Bigr) = \sum_{\,l\in\mathsf S}\tr\bigl(\Omega_{\emptyset}^{-2}\Delta_{l}\bigr) \nonumber \\
    & = \sum_{\,l\in\mathsf S} r_{l}\bigl(\Omega_{\emptyset}^{2}\bigr).
\end{align*}
This proves equation~\eqref{eq:quadratic_expansion}.  

To bound the second‐order term, define $M \,\triangleq\, A^{-1/2}\,B\,A^{-1/2}$, so that $M\succeq 0$ and $B = A^{1/2}\,M\,A^{1/2}$. Then
\begin{align*}
    & \tr\bigl(A^{-1}BA^{-1}BA^{-1}\bigr) \nonumber \\
    &= \tr\bigl(A^{-1/2}\,(A^{-1/2}BA^{-1/2})\,(A^{-1/2}BA^{-1/2})\,A^{-1/2}\bigr) \nonumber \\
    &= \tr\bigl(M^{2}\,A^{-1}\bigr).
\end{align*}
Because $A^{-1}\succeq 0$, we have
\[
    \tr\bigl(M^{2}A^{-1}\bigr) \;\le\; \|A^{-1}\|_{2}\,\tr\bigl(M^{2}\bigr) \;=\; \|A^{-1}\|_{2}\,\|M\|_{F}^{2}.
\]
Moreover,
\begin{align*}
    \|M\|_{F} & = \bigl\|A^{-1/2}\,B\,A^{-1/2}\bigr\|_{F} \nonumber \\
    & \le\; \bigl\|A^{-1/2}\bigr\|_{2}^{2}\,\|B\|_{F} \;=\;  \|A^{-1}\|_{2}\,\|\Delta_{\mathsf S}\|_{F}. 
\end{align*}
Combining gives
\begin{align*}
    \tr\bigl(A^{-1}BA^{-1}BA^{-1}\bigr) \;& \leq \;  \|A^{-1}\|_{2}\,\bigl(\|A^{-1}\|_{2}\,\|\Delta_{\mathsf S}\|_{F}\bigr)^{2} \nonumber \\
    & =\;  \|A^{-1}\|_{2}^{3}\,\|\Delta_{\mathsf S}\|_{F}^{2}.
\end{align*}
Since $\|\Delta_{\mathsf S}\|_{F}\le\sum_{l\in\mathsf S}\|\Delta_{l}\|_{F}\le \kappa\,\zeta$, we obtain the bound in~\eqref{eq:quadratic_bound}.  This completes the proof.

\bibliography{bibliography}

\begin{thebibliography}{10}
\providecommand{\url}[1]{#1}
\csname url@samestyle\endcsname
\providecommand{\newblock}{\relax}
\providecommand{\bibinfo}[2]{#2}
\providecommand{\BIBentrySTDinterwordspacing}{\spaceskip=0pt\relax}
\providecommand{\BIBentryALTinterwordstretchfactor}{4}
\providecommand{\BIBentryALTinterwordspacing}{\spaceskip=\fontdimen2\font plus
\BIBentryALTinterwordstretchfactor\fontdimen3\font minus
  \fontdimen4\font\relax}
\providecommand{\BIBforeignlanguage}[2]{{%
\expandafter\ifx\csname l@#1\endcsname\relax
\typeout{** WARNING: IEEEtran.bst: No hyphenation pattern has been}%
\typeout{** loaded for the language `#1'. Using the pattern for}%
\typeout{** the default language instead.}%
\else
\language=\csname l@#1\endcsname
\fi
#2}}
\providecommand{\BIBdecl}{\relax}
\BIBdecl

\bibitem{yang2018grand}
G.-Z. Yang, J.~Bellingham, P.~E. Dupont, P.~Fischer, L.~Floridi, R.~Full,
  N.~Jacobstein, V.~Kumar, M.~McNutt, R.~Merrifield \emph{et~al.}, ``The grand
  challenges of science robotics,'' \emph{Science Robotics}, vol.~3, no.~14, p.
  eaar7650, 2018.

\bibitem{sala2006landmark}
P.~Sala, R.~Sim, A.~Shokoufandeh, and S.~Dickinson, ``Landmark selection for
  vision-based navigation,'' \emph{IEEE Transactions on Robotics}, vol.~22,
  no.~2, pp. 334--349, 2006.

\bibitem{lerner2007landmark}
R.~Lerner, E.~Rivlin, and I.~Shimshoni, ``Landmark selection for task-oriented
  navigation,'' \emph{IEEE Transactions on Robotics}, vol.~23, no.~3, pp.
  494--505, 2007.

\bibitem{mu2017two}
B.~Mu, L.~Paull, A.-A. Agha-Mohammadi, J.~J. Leonard, and J.~P. How,
  ``Two-stage focused inference for resource-constrained minimal collision
  navigation,'' \emph{IEEE Transactions on Robotics}, vol.~33, no.~1, pp.
  124--140, 2017.

\bibitem{carlone2018attention}
L.~Carlone and S.~Karaman, ``Attention and anticipation in fast visual-inertial
  navigation,'' \emph{IEEE Transactions on Robotics}, vol.~35, no.~1, pp.
  1--20, 2018.

\bibitem{mousavi2020estimation}
H.~K. Mousavi and N.~Motee, ``Estimation with fast feature selection in robot
  visual navigation,'' \emph{IEEE Robotics and Automation Letters}, vol.~5,
  no.~2, pp. 3572--3579, 2020.

\bibitem{pandey2024scalable}
V.~Pandey, A.~Amini, G.~Liu, U.~Topcu, Q.~Sun, K.~Daniilidis, and N.~Motee,
  ``Scalable networked feature selection with randomized algorithm for robot
  navigation,'' \emph{arXiv preprint arXiv:2403.12279}, 2024.

\bibitem{davison2005active}
A.~J. Davison, ``Active search for real-time vision,'' in \emph{Proceedings of
  the 10th IEEE International Conference on Computer Vision (ICCV)},
  vol.~1.\hskip 1em plus 0.5em minus 0.4em\relax IEEE, 2005, pp. 66--73.

\bibitem{strasdat2009landmark}
H.~Strasdat, C.~Stachniss, and W.~Burgard, ``Which landmark is useful? learning
  selection policies for navigation in unknown environments,'' in
  \emph{Proceedings of the 2009 IEEE International Conference on Robotics and
  Automation (ICRA)}.\hskip 1em plus 0.5em minus 0.4em\relax IEEE, 2009, pp.
  1410--1415.

\bibitem{siami2017growing}
M.~Siami and N.~Motee, ``Growing linear dynamical networks endowed by spectral
  systemic performance measures,'' \emph{IEEE Transactions on Automatic
  Control}, vol.~63, no.~7, pp. 2091--2106, 2017.

\bibitem{siami2018network}
------, ``Network abstraction with guaranteed performance bounds,'' \emph{IEEE
  Transactions on Automatic Control}, vol.~63, no.~10, pp. 3301--3316, 2018.

\bibitem{siami2020deterministic}
M.~Siami, A.~Olshevsky, and A.~Jadbabaie, ``Deterministic and randomized
  actuator scheduling with guaranteed performance bounds,'' \emph{IEEE
  Transactions on Automatic Control}, vol.~66, no.~4, pp. 1686--1701, 2020.

\bibitem{siami2020separation}
M.~Siami and A.~Jadbabaie, ``A separation theorem for joint sensor and actuator
  scheduling with guaranteed performance bounds,'' \emph{Automatica}, vol. 119,
  p. 109054, 2020.

\bibitem{nemhauser1978analysis}
G.~L. Nemhauser, L.~A. Wolsey, and M.~L. Fisher, ``An analysis of
  approximations for maximizing submodular set functions—i,''
  \emph{Mathematical Programming}, vol.~14, pp. 265--294, 1978.

\bibitem{badanidiyuru2014fast}
A.~Badanidiyuru and J.~Vondr{\'a}k, ``Fast algorithms for maximizing submodular
  functions,'' in \emph{Proceedings of the 25th Annual ACM-SIAM Symposium on
  Discrete Algorithms (SODA)}.\hskip 1em plus 0.5em minus 0.4em\relax SIAM,
  2014, pp. 1497--1514.

\bibitem{mirzasoleiman2015lazier}
B.~Mirzasoleiman, A.~Badanidiyuru, A.~Karbasi, J.~Vondr{\'a}k, and A.~Krause,
  ``Lazier than lazy greedy,'' in \emph{Proceedings of the AAAI Conference on
  Artificial Intelligence}, vol.~29, no.~1, 2015.

\bibitem{buchbinder2024deterministic}
N.~Buchbinder and M.~Feldman, ``Deterministic algorithm and faster algorithm
  for submodular maximization subject to a matroid constraint,'' \emph{arXiv
  preprint arXiv:2408.03583}, 2024.

\bibitem{das2011submodular}
A.~Das and D.~Kempe, ``Submodular meets spectral: Greedy algorithms for subset
  selection, sparse approximation and dictionary selection,'' in
  \emph{Proceedings of the 28th International Conference on Machine Learning
  (ICML)}.\hskip 1em plus 0.5em minus 0.4em\relax Omnipress, 2011, pp.
  1057--1064.

\bibitem{summers2019performance}
T.~Summers and M.~Kamgarpour, ``Performance guarantees for greedy maximization
  of non-submodular controllability metrics,'' in \emph{Proceedings of the 18th
  European Control Conference (ECC)}.\hskip 1em plus 0.5em minus 0.4em\relax
  IEEE, 2019, pp. 2796--2801.

\bibitem{hashemi2018randomized}
A.~Hashemi, M.~Ghasemi, H.~Vikalo, and U.~Topcu, ``A randomized greedy
  algorithm for near-optimal sensor scheduling in large-scale sensor
  networks,'' in \emph{Proceedings of the 2018 American Control Conference
  (ACC)}.\hskip 1em plus 0.5em minus 0.4em\relax IEEE, 2018, pp. 1027--1032.

\bibitem{zhao2020good}
Y.~Zhao and P.~A. Vela, ``Good feature matching: Toward accurate, robust
  vo/vslam with low latency,'' \emph{IEEE Transactions on Robotics}, vol.~36,
  no.~3, pp. 657--675, 2020.

\bibitem{hesch2014camera}
J.~A. Hesch, D.~G. Kottas, S.~L. Bowman, and S.~I. Roumeliotis,
  ``Camera-imu-based localization: Observability analysis and consistency
  improvement,'' \emph{The International Journal of Robotics Research},
  vol.~33, no.~1, pp. 182--201, 2014.

\bibitem{leutenegger2015keyframe}
S.~Leutenegger, S.~Lynen, M.~Bosse, R.~Siegwart, and P.~Furgale,
  ``Keyframe-based visual--inertial odometry using nonlinear optimization,''
  \emph{The International Journal of Robotics Research}, vol.~34, no.~3, pp.
  314--334, 2015.

\bibitem{sibley2010sliding}
G.~Sibley, L.~Matthies, and G.~Sukhatme, ``Sliding window filter with
  application to planetary landing,'' \emph{Journal of Field Robotics},
  vol.~27, no.~5, pp. 587--608, 2010.

\bibitem{forster2015imu}
C.~Forster, L.~Carlone, F.~Dellaert, and D.~Scaramuzza, ``Imu preintegration on
  manifold for efficient visual-inertial maximum-a-posteriori estimation,''
  \emph{Georgia Institute of Technology Technical Report}, 2015.

\bibitem{lupton2011visual}
T.~Lupton and S.~Sukkarieh, ``Visual-inertial-aided navigation for high-dynamic
  motion in built environments without initial conditions,'' \emph{IEEE
  Transactions on Robotics}, vol.~28, no.~1, pp. 61--76, 2011.

\bibitem{campos2021orb}
C.~Campos, R.~Elvira, J.~J.~G. Rodríguez, J.~M.~M. Montiel, and J.~D. Tardós,
  ``Orb-slam3: An accurate open-source library for visual, visual--inertial,
  and multimap slam,'' \emph{IEEE Transactions on Robotics}, vol.~37, no.~6,
  pp. 1874--1890, 2021.

\bibitem{bian2017guarantees}
A.~A. Bian, J.~M. Buhmann, A.~Krause, and S.~Tschiatschek, ``Guarantees for
  greedy maximization of non-submodular functions with applications,'' in
  \emph{Proceedings of the 34th International Conference on Machine Learning
  (ICML)}.\hskip 1em plus 0.5em minus 0.4em\relax PMLR, 2017, pp. 498--507.

\bibitem{hashemi2020randomized}
A.~Hashemi, M.~Ghasemi, H.~Vikalo, and U.~Topcu, ``Randomized greedy sensor
  selection: Leveraging weak submodularity,'' \emph{IEEE Transactions on
  Automatic Control}, vol.~66, no.~1, pp. 199--212, 2020.

\bibitem{burri2016euroc}
M.~Burri, J.~Nikolic, P.~Gohl, T.~Schneider, J.~Rehder, S.~Omari, M.~W.
  Achtelik, and R.~Siegwart, ``The euroc micro aerial vehicle datasets,''
  \emph{The International Journal of Robotics Research}, vol.~35, no.~10, pp.
  1157--1163, 2016.

\bibitem{QCar_Quanser}
Q.~Inc., ``Qcar: Sensor-rich autonomous vehicle for self-driving
  applications,'' \url{https://www.quanser.com/products/qcar/#details}, 2024.

\bibitem{olshevsky2017non}
A.~Olshevsky, ``On (non) supermodularity of average control energy,''
  \emph{IEEE Transactions on Control of Network Systems}, vol.~5, no.~3, pp.
  1177--1181, 2017.

\bibitem{vafaee2024learning}
R.~Vafaee and M.~Siami, ``Learning-based sparse sensing with performance
  guarantees,'' \emph{IEEE Transactions on Automatic Control}, 2024.

\bibitem{vafaee2023exploring}
------, ``Exploring non-submodular scheduling for large-scale sensor
  networks,'' \emph{IEEE Control Systems Letters}, 2023.

\bibitem{chamon2020approximate}
L.~F.~O. Chamon, G.~J. Pappas, and A.~Ribeiro, ``Approximate supermodularity of
  kalman filter sensor selection,'' \emph{IEEE Transactions on Automatic
  Control}, vol.~66, no.~1, pp. 49--63, 2020.

\bibitem{carlone2014eliminating}
L.~Carlone, Z.~Kira, C.~Beall, V.~Indelman, and F.~Dellaert, ``Eliminating
  conditionally independent sets in factor graphs: A unifying perspective based
  on smart factors,'' in \emph{Proceedings of the 2014 IEEE International
  Conference on Robotics and Automation (ICRA)}.\hskip 1em plus 0.5em minus
  0.4em\relax IEEE, 2014, pp. 4290--4297.

\bibitem{williamson2011design}
D.~P. Williamson and D.~B. Shmoys, \emph{The design of approximation
  algorithms}.\hskip 1em plus 0.5em minus 0.4em\relax Cambridge university
  press, 2011.

\bibitem{borwein2006convex}
J.~Borwein and A.~Lewis, \emph{Convex Analysis}.\hskip 1em plus 0.5em minus
  0.4em\relax Springer, 2006.

\bibitem{qin2018vins}
T.~Qin, P.~Li, and S.~Shen, ``Vins-mono: A robust and versatile monocular
  visual-inertial state estimator,'' \emph{IEEE Transactions on Robotics},
  vol.~34, no.~4, pp. 1004--1020, 2018.

\bibitem{shi1994good}
J.~Shi and C.~Tomasi, ``Good features to track,'' in \emph{Proceedings of IEEE
  Conference on Computer Vision and Pattern Recognition (CVPR)}.\hskip 1em plus
  0.5em minus 0.4em\relax IEEE, 1994, pp. 593--600.

\bibitem{lucas1981iterative}
B.~D. Lucas and T.~Kanade, ``An iterative image registration technique with an
  application to stereo vision,'' in \emph{IJCAI'81: Proceedings of the 7th
  International Joint Conference on Artificial Intelligence}, vol.~2, 1981, pp.
  674--679.

\bibitem{zhang2015sensor}
H.~Zhang, R.~Ayoub, and S.~Sundaram, ``Sensor selection for optimal filtering
  of linear dynamical systems: Complexity and approximation,'' in
  \emph{Proceedings of the 54th IEEE Conference on Decision and Control
  (CDC)}.\hskip 1em plus 0.5em minus 0.4em\relax IEEE, 2015, pp. 5002--5007.

\bibitem{zhang2018tutorial}
Z.~Zhang and D.~Scaramuzza, ``A tutorial on quantitative trajectory evaluation
  for visual (-inertial) odometry,'' in \emph{2018 IEEE/RSJ International
  Conference on Intelligent Robots and Systems (IROS)}.\hskip 1em plus 0.5em
  minus 0.4em\relax IEEE, 2018, pp. 7244--7251.

\bibitem{zhang2021balancing}
L.~Zhang, D.~Wisth, M.~Camurri, and M.~Fallon, ``Balancing the budget: Feature
  selection and tracking for multi-camera visual-inertial odometry,''
  \emph{IEEE Robotics and Automation Letters}, vol.~7, no.~2, pp. 1182--1189,
  2021.

\bibitem{boyd2004convex}
S.~P. Boyd and L.~Vandenberghe, \emph{Convex Optimization}.\hskip 1em plus
  0.5em minus 0.4em\relax Cambridge University Press, 2004.

\end{thebibliography}

\begin{IEEEbiography}[{\includegraphics[width=1in,height=1.25in,clip,keepaspectratio]{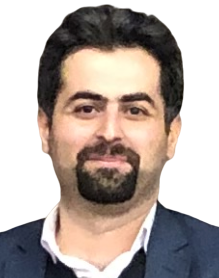}}]{Reza Vafaee}
(Member, IEEE) received the second master’s degree in pure and applied mathematics (with highest distinction) from Montclair State University, Montclair, NJ, USA, in 2020, and the Ph.D. degree in electrical engineering from Northeastern University, Boston, MA, USA, in 2024.  
He is currently a Postdoctoral Fellow in the Department of Computer Science at Boston College, Chestnut Hill, MA, USA.  
\\His research interests include the mathematical and algorithmic foundations of network optimization and control with applications in robotics and multi-agent systems, estimation theory, and spectral graph theory.
\end{IEEEbiography}

\begin{IEEEbiography}[{\includegraphics[width=1in,height=1.25in,clip,keepaspectratio]{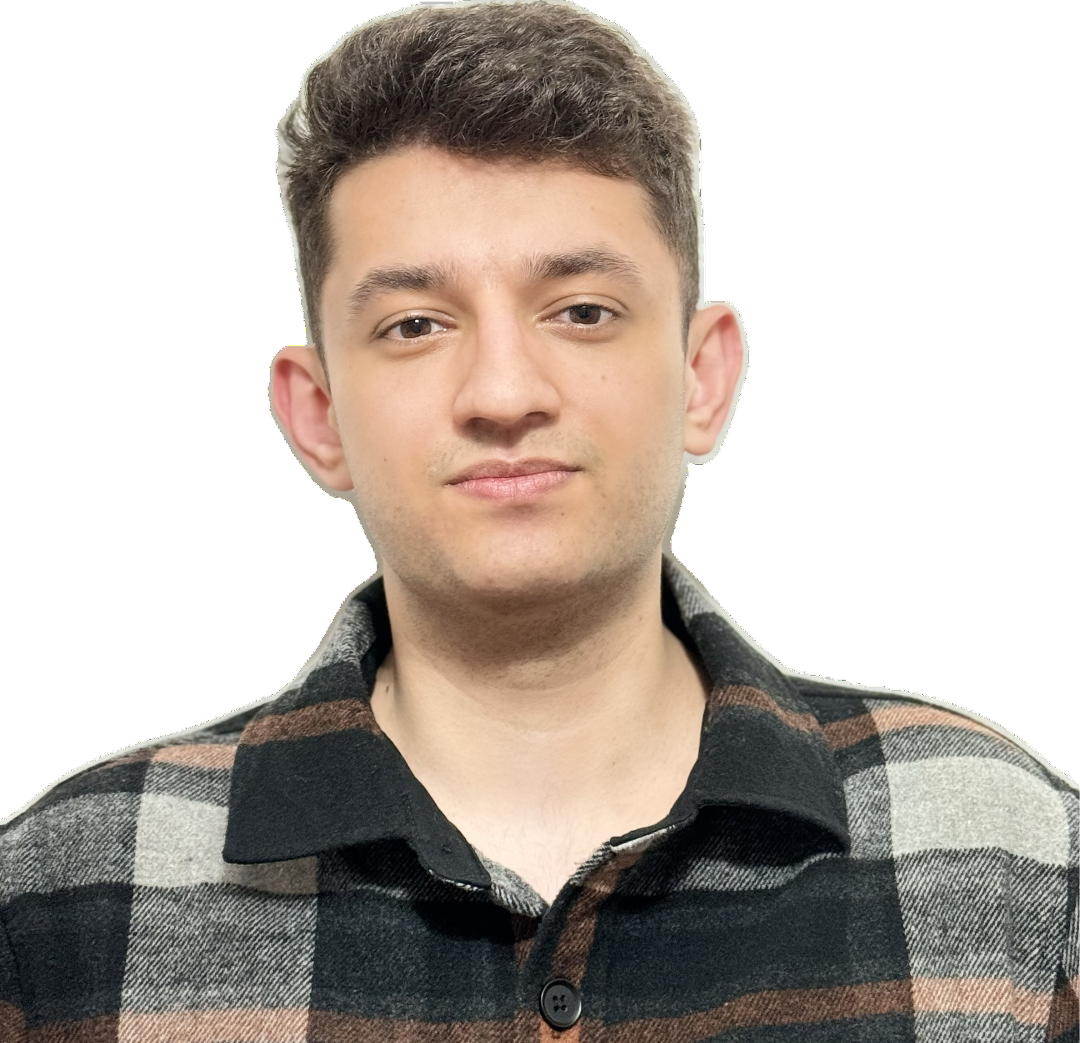}}]{Kian Behzad}
(Graduate Student Member, IEEE) received the B.Sc. degree in Electrical Engineering from Amirkabir University of Technology, Tehran, Iran. He is currently working toward the Ph.D. degree in Electrical Engineering at Northeastern University (NEU), Boston, MA, USA, where he is a research assistant in the Siami Lab.\\
He is conducting research focusing on robotics, WiFi sensing, and control systems, with applications in autonomous systems and wireless sensing. His research interests include visual–inertial navigation, optimization, and active sensing.
\end{IEEEbiography}

\begin{IEEEbiography}[{\includegraphics[width=1in,height=1.25in,clip,keepaspectratio]{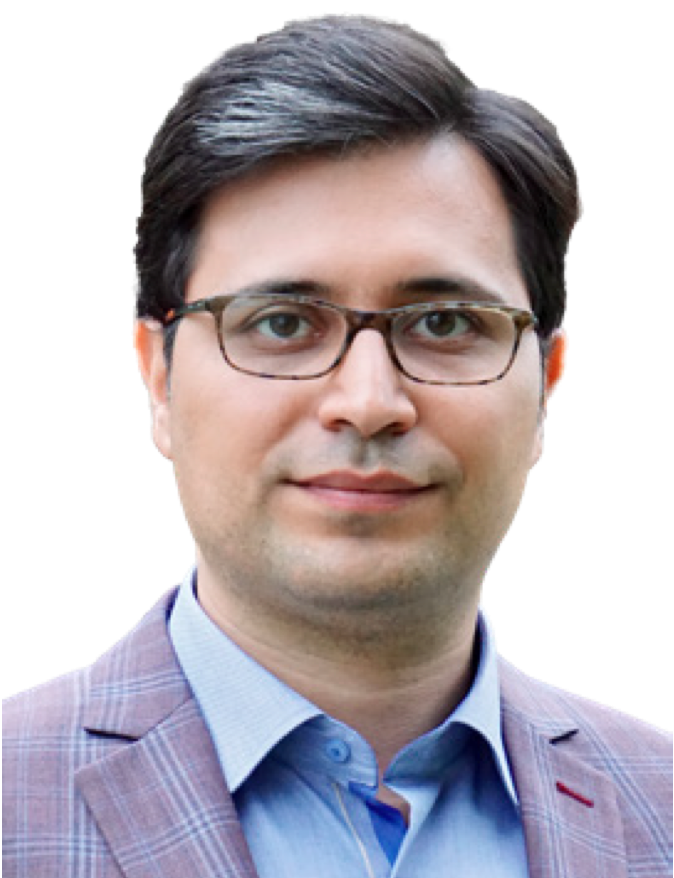}}]{Milad Siami} (Senior Member, IEEE) received the B.Sc. degrees in electrical engineering and pure mathematics, and the M.Sc. degree in electrical engineering from Sharif University of Technology, Tehran, Iran, in 2009 and 2011, respectively, and the M.Sc. and Ph.D. degrees in mechanical engineering from Lehigh University, Bethlehem, PA, USA, in 2014 and 2017. He was a Postdoctoral Associate at the Massachusetts Institute of Technology from 2017 to 2019. \\
He is currently an Associate Professor with the Department of Electrical and Computer Engineering, Northeastern University, Boston, MA, USA. His research interests include distributed control, distributed optimization, and sparse sensing in robotics and cyber–physical systems. 
\end{IEEEbiography}

\begin{IEEEbiography}[{\includegraphics[width=1in,height=1.25in,clip,keepaspectratio, trim=50 0 50 0]{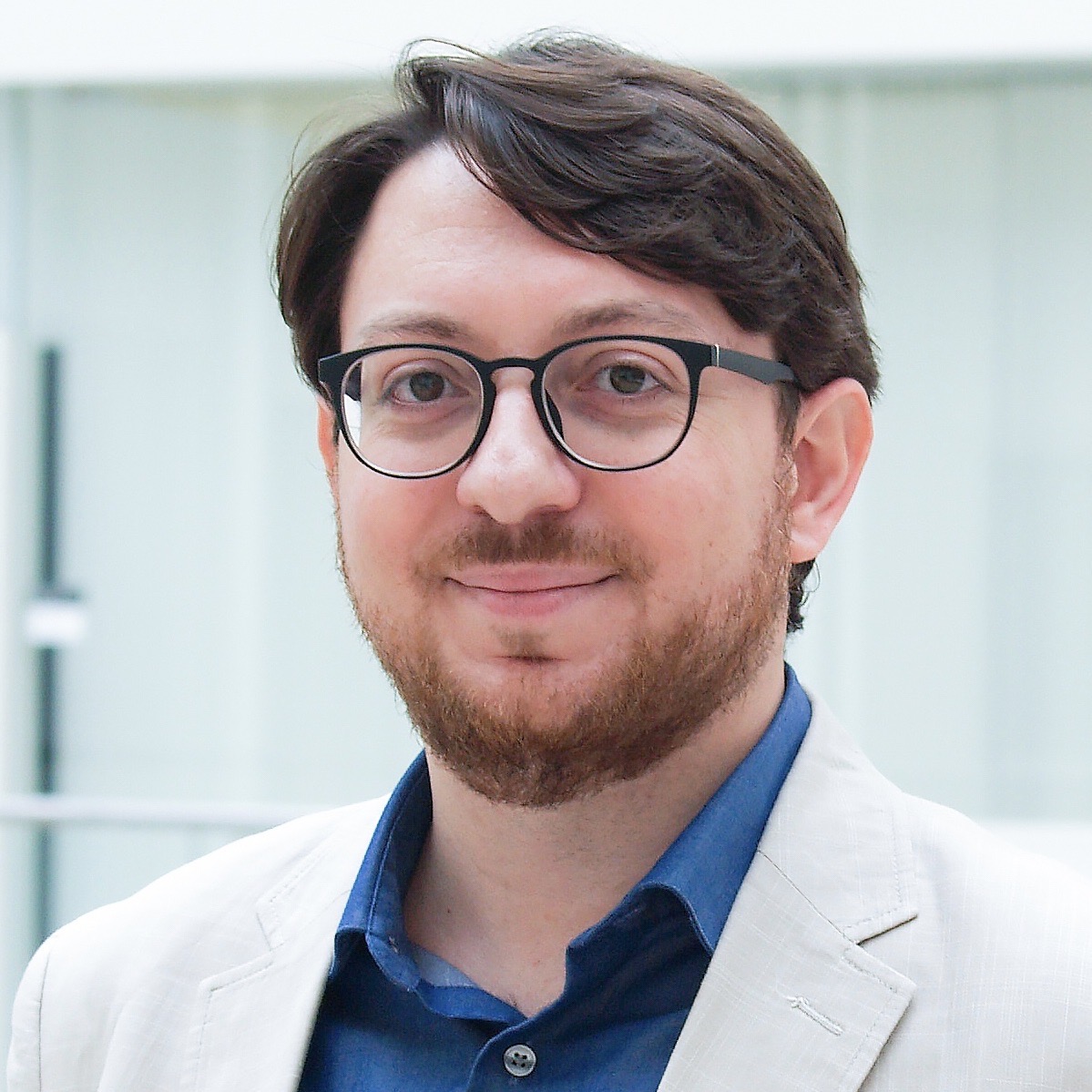}}]
{Luca Carlone} (Senior Member, IEEE) received the Ph.D. degree in robotics from the Polytechnic University of Turin, Turin, Italy, in 2012. He is currently an Associate Professor with the Department of Aeronautics and Astronautics, Massachusetts Institute of Technology, Cambridge, MA, USA, where he is also a Principal Investigator with the Laboratory for Information and Decision Systems (LIDS).\\ His research interests include estimation, optimization, and learning, with applications to robot perception and decision making. He has received several honors, including the NSF CAREER Award, the RSS Early Career Award, a Sloan Research Fellowship, and multiple Best Paper Awards at leading robotics conferences.
\end{IEEEbiography}

\begin{IEEEbiography}[{\includegraphics[width=1in,height=1.25in,clip,keepaspectratio, trim=.2in .1in .2in .1in]{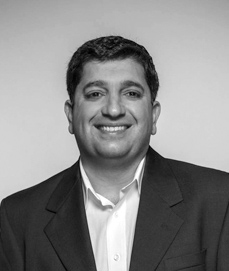}}]{Ali Jadbabaie} (Fellow, IEEE) received the Ph.D. degree in control and dynamical systems from the California Institute of Technology, Pasadena, CA, USA, in 2001. He is currently the JR East Professor and Head of the Department of Civil and Environmental Engineering at the Massachusetts Institute of Technology, Cambridge, MA, USA, where he is also a core faculty member with the Institute for Data, Systems, and Society and a Principal Investigator with the Laboratory for Information and Decision Systems (LIDS). \\
His research interests include networked dynamical systems, distributed optimization, and multiagent coordination and control. He is a Fellow of IEEE and a recipient of the Vannevar Bush Faculty Fellowship, NSF CAREER Award, and the George S. Axelby Best Paper Award.
\end{IEEEbiography}

 \end{document}